\documentclass[twoside]{article}

\usepackage{booktabs}       % professional-quality tables
\usepackage{amsfonts}       % blackboard math symbols
\usepackage{nicefrac}       % compact symbols for 1/2, etc.
\usepackage{microtype}      % microtypography
\usepackage{mathtools}
\usepackage{multirow}
\usepackage{bm}
\usepackage{epsfig}
\usepackage{graphicx}
\usepackage{caption}
\usepackage{subcaption}
\usepackage{amsthm}
\usepackage{amsmath}
\usepackage{amssymb}
\usepackage{enumitem}
\usepackage{makecell}
\usepackage{wrapfig}
\usepackage{indentfirst}
\usepackage{verbatim}
\usepackage{color}
\usepackage{setspace}
\usepackage{array}
\usepackage{booktabs}
\usepackage{stackengine}

\newtheorem{theorem}{Theorem}
\newtheorem{lemma}{Lemma}
\newtheorem{proposition}{Proposition}
\newcommand{\norm}[1]{\left\lVert#1\right\rVert}

\newcommand{\eg}{\emph{e.g.}}
\newcommand{\ie}{\emph{i.e.}}
\newcommand*\mystrut[1]{\vrule width0pt height0pt depth#1\relax}
\renewcommand{\captionlabelfont}{\scriptsize}

\usepackage[pagebackref=true,breaklinks=true,colorlinks,bookmarks=false]{hyperref}
\hypersetup{linkcolor=[rgb]{0.8551,0.2333,0.2333}}
\hypersetup{citecolor=[rgb]{0.3333,0.2333,0.7551}}
\usepackage{url}            % simple URL typesetting

\usepackage[accepted]{aistats2021}

\begin{document}

\twocolumn[

\aistatstitle{Learning with Hyperspherical Uniformity}
\vspace{-1.8mm}
\aistatsauthor{\small Weiyang Liu\textsuperscript{1,2,*}~~Rongmei Lin\textsuperscript{3,*}~~Zhen Liu\textsuperscript{4,*}~~Li Xiong\textsuperscript{3}~~Bernhard Schölkopf\textsuperscript{2}~~Adrian Weller\textsuperscript{1,5}}
\vspace{1.4mm}
\aistatsaddress{\small \textsuperscript{1}University of Cambridge~~~\textsuperscript{2}MPI-IS Tübingen~~~\textsuperscript{3}Emory University~~~\textsuperscript{4}Université de Montréal~~~\textsuperscript{5}Alan Turing Institute}]

\begin{abstract}
\vspace{-1.6mm}
  Due to the over-parameterization nature, neural networks are a powerful tool for nonlinear function approximation. In order to achieve good generalization on unseen data, a suitable inductive bias is of great importance for neural networks. One of the most straightforward ways is to regularize the neural network with some additional objectives. $\ell_2$ regularization serves as a standard regularization for neural networks. Despite its popularity, it essentially regularizes one dimension of the individual neuron, which is not strong enough to control the capacity of highly over-parameterized neural networks. Motivated by this, hyperspherical uniformity is proposed as a novel family of relational regularizations that impact the interaction among neurons. We consider several geometrically distinct ways to achieve hyperspherical uniformity. The effectiveness of hyperspherical uniformity is justified by theoretical insights and empirical evaluations.
\end{abstract}

\vspace{-2.9mm}
\section{Introduction}
\vspace{-1.6mm}

As one of the most effective ways to control the capacity of over-parameterized neural networks, regularization serves an important role to prevent overfitting and improve generalization. Current prevailing weight regularizations can be divided into two major categories: \emph{Individual regularization} and \emph{relational regularization}. Taking $n$ $d$-dimensional neurons $\bm{w}_1,\cdots,\bm{w}_n\in\mathbb{R}^d$ from one layer of the neural network as an example, we typically can have the following regularization:
\begin{equation*}
\footnotesize
   \mathcal{L}_{\text{reg}} = \underbrace{\lambda_{\text{I}}\cdot\sum\nolimits_{i=1}^n h(\bm{w}_i)}_{\text{Individual Regularization}} +  \underbrace{ \mystrut{1.3ex}\lambda_{\textnormal{R}}\cdot g(\bm{w}_1,\cdots,\bm{w}_n)}_{\text{Relational Regularization}}
\end{equation*}
where $h$ defines a function that independently describes some properties of the individual weight $\bm{w}_i$ (\eg, $\ell_p$ norm), and $g$ is a function that characterizes the interaction between different weights $\bm{w}_1,\cdots,\bm{w}_n$ (\eg, orthogonality). Individual regularization (\eg, weight decay) is typically used by default in neural networks, while relational regularization tends to be overlooked. Individual regularization essentially regularizes only one dimension of the weights (from the view of spherical coordinate systems~\cite{blumenson1960derivation}), which is not strong enough for highly over-parameterized neural networks. In contrast, relational regularization can be viewed as regularizing $d-1$ dimensions of the weights, and encodes stronger inductive bias with relational information. Individual and relational regularizations usually serve complimentary roles to each other and can be used simultaneously. In this paper, we study the importance of a suitable relational regularization for over-parameterized models (\eg, neural networks) and explore a novel family of the relational regularizations -- \emph{hyperspherical uniformity}.

Hyperspherical uniformity characterizes the diversity of vectors on a unit hypersphere. Different from orthogonality where perpendicular vectors are defined to be diverse, hyperspherical uniformity encourages vectors to be spaced apart with as large an angle as possible such that these vectors can be uniformly distributed over the hypersphere. In order to promote hyperspherical uniformity with an explicit regularization, we formulate several distinct learning objectives that are conceptually appealing and geometrically interpretable. Specifically, we consider minimum hyperspherical energy~(MHE)~\cite{liu2018learning}, maximum hyperspherical separation (MHS), maximum hyperspherical polarization~(MHP), minimum hyperspherical covering~(MHC), and maximum Gram determinant~(MGD). Different regularization objectives yield distinct geometric interpretations and optimization dynamics. Moreover, we draw inspiration from statistical uniformity testing on the hypersphere and provide a novel and unified view on understanding these learning objectives.

The motivation to encourage hyperspherical uniformity lies in three aspects. First, we argue that hyperspherical uniformity leads to better optimization and generalization. \cite{xie2016diverse} proves that optimizing one-hidden-layer neural networks with hyperspherically uniform neurons has no spurious local minima. \cite{liu2018learning,Lin20CoMHE} empirically show that promoting hyperspherical uniformity of neurons can effectively improve the generalization of neural networks.  Interestingly, hyperspherical uniformity implicitly regularizes the neurons to be close to the initialization, partially implementing Occam’s razor to keep neural networks as simple as possible (see Section~\ref{theoretical_insights}). Second, hyperspherical uniformity can remove neuron redundancy and encourage the neurons to be diverse on the hypersphere. In the light of \cite{shang2016understanding} that shows deeply learned neurons are highly redundant, hyperspherical uniformity can serve as a useful regularization to remove such redundancy. Third, hyperspherical uniformity has a clear geometric interpretation and theoretical merits. There exists a close connection between hyperspherical uniformity and orthogonality. The effectiveness of hyperspherical uniformity can also be justified from multiple theoretical viewpoints. Our contribution can be summarized as follows:

\vspace{-2.9mm}
\begin{itemize}[leftmargin=*]
\setlength\itemsep{0.04em}
    \item We introduce a general property -- hyperspherical uniformity as a regularizer for neural networks.
    \item To achieve hyperspherical uniformity, we propose several well-performing approaches (MHC, MHS, MHP, MGD) with distinct geometric interpretations.
    \item We establish a connection between hyperspherical uniformity and orthogonality, showing that hyperspherical uniformity is a more general property.
    \item We provide some insights and discussions on the geometric and spectral properties and regularization effects of hyperspherical uniformity.
    \item We apply hyperspherical uniformity to a number of applications and demonstrate superior performance over existing regularizations such as orthogonality.
\end{itemize}

\vspace{-3.8mm}
\section{Related Work}
\vspace{-2.4mm}

\textbf{Relational regularizations}. There are quite a number of relational regularizations that have been used in neural networks, such as orthogonality regularization~\cite{bansal2018can,liu2017hyper,rodriguez2016regularizing,huang2018orthogonal,choromanski2018initialization}, unitary constraint~\cite{jing2017tunable,wisdom2016full,arjovsky2016unitary},  decorrelation~\cite{rodriguez2016regularizing,cogswell2015reducing,xie2017uncorrelation}, spectral regularization~\cite{yoshida2017spectral}, low-rank regularization~\cite{tai2015convolutional}, angular constraint~\cite{xie2017learning,li2017improving}, etc. Most of these relational regularizations are either directly based on orthogonality or based on some notions related to orthogonality (\eg, correlation). Quite differently, hyperspherical uniformity encourages neurons to be uniformly distributed over the hypersphere.

\vspace{-0.6mm}

\textbf{Hyperspherical learning}.\! \cite{liu2016large,liu2017hyper,liu2018decoupled,LiuNIPS19,liu2017sphereface,wang2018cosface,wang2018additive,deng2019arcface,davidson2018hyperspherical,park2019sphere,Chen20AVH,Liu2020OPT} propose to learn representations on hypersphere and show that angular information in neural networks, in contrast to magnitude, preserves the key semantics and is very crucial to generalization. \cite{liu2018learning} regularizes the diversity of neurons on hypersphere by minimizing their pairwise energy. \cite{Lin20CoMHE} explores how projections can help to better minimize such energy.

\section{Connection between Hyperspherical Uniformity and Orthogonality}
\vspace{-2mm}

Before we discuss specific methods to achieve hyperspherical uniformity, we first reveal an interesting connection between hyperspherical uniformity and orthogonality with the following theoretical statement:

\vspace{1.6mm}
\begin{theorem}\label{ortho_uniform}
For $\epsilon,\delta>0$ and $d\geq \max(4,\delta^{-2}\epsilon^{-1})$, every orthonormal basis $\bm{V}=\{\bm{v}_1,\cdots,\bm{v}_d\}$ in $\mathbb{R}^d$ is $\epsilon$-$\delta$-uniform distributed on the unit hypersphere $\mathbb{S}(\mathbb{R}^d)$, i.e., for every Borel set $\bm{A}\subseteq\mathbb{S}(\mathbb{R}^d)$, we have
\begin{equation}
\footnotesize
    \mathbb{P}\bigg(\bigg{|}\frac{\textnormal{Card}(\bm{V}\cap\bm{R}(\bm{A}))}{\textnormal{Card}(\bm{V})}-u(\bm{A})\bigg{|}\leq\delta\bigg)\geq 1-\epsilon
\end{equation}
where $\textnormal{Card}(\cdot)$ denotes the cardinality of a set, $\bm{R}(\bm{A})$ is a random rotation of $\bm{A}$, and $u$ denotes the uniform probability measure over $\mathbb{S}(\mathbb{R}^d)$.
\end{theorem}
\vspace{-0.1mm}

Theorem~\ref{ortho_uniform} can be obtained from \cite{goldstein2014any} and is related to concentration of measure~\cite{milman2009asymptotic} and Raz's lemma~\cite{raz1999exponential}. It shows that uniformly sampling a basis from the orthogonal group in high dimensions gives an approximate uniform distribution on $\mathbb{S}(\mathbb{R}^d)$, implying that any orthogonal basis is approximately uniformly distributed over the hypersphere in high dimensions. This result bridges hyperspherical uniformity and orthogonality, implying hyperspherical uniformity is more general.

\setlength{\columnsep}{8pt}
\begin{wrapfigure}{r}{0.181\textwidth}
  \begin{center}
  \advance\leftskip+1mm
  \renewcommand{\captionlabelfont}{\scriptsize}
    \vspace{-0.27in}  
    \includegraphics[width=0.179\textwidth]{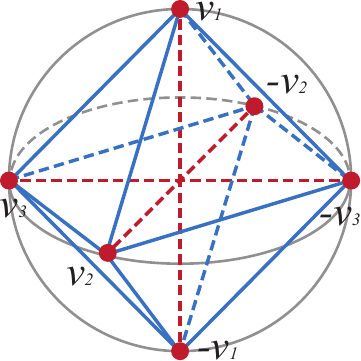}
    \vspace{-0.29in} 
    \caption{\scriptsize Cross-polytope.}\label{crossp}
    \vspace{-0.2in} 
  \end{center}
\end{wrapfigure}

More interestingly, enforcing hyperspherical uniformity for $2d+2$ vectors in $\mathbb{S}^d$ leads to a cross-polytope~\cite{yudin1992minimum}. Fig.~\ref{crossp} presents a 3-dimensional cross-polytope in $\mathbb{S}^2$ (with 6 vectors). We assume that there exists a unit-vector set with $d+1$ vectors in total: $\{\bm{v}_1,\cdots,\bm{v}_{d+1}\in\mathbb{S}^{d}\}$. Then we construct a new set with $2d+2$ vectors by adding all vectors with opposite direction to the original set: $\{\bm{v}_1,\cdots,\bm{v}_{d+1},-\bm{v}_1,\cdots,-\bm{v}_{d+1}\in\mathbb{S}^{d}\}$. Promoting hyperspherical uniformity for these $2d+2$ vectors (in the new set) is actually equivalent to promoting pairwise orthogonality among the original $d+1$ vectors (in the original set). This equivalence builds a strong connection between hyperspherical uniformity and orthogonality. We will discuss this further in Section~\ref{optimality}.

\vspace{-2mm}
\section{Towards Hyperspherical Uniformity}
\vspace{-1.5mm}
We design multiple learning objectives to achieve hyperspherical uniformity and discuss their close connections. MHE was first proposed in our previous work~\cite{liu2018learning}, while the others are part of the contributions in this work.

\vspace{-2mm}
\subsection{Minimum Hyperspherical Energy}
\vspace{-1.5mm}

Inspired by the Thomson problem~\cite{thomson1904xxiv} where one seeks to find an equilibrium state with minimum potential energy that distributes $N$ electrons on a unit sphere as evenly as possible, MHE~\cite{liu2018learning,Lin20CoMHE} encourages hyperspherical diversity and defines the following learning objective for $n$ $d$-dimensional vectors (\eg, neurons)  $\bm{W}_n=\{\bm{w}_1,\cdots,\bm{w}_n\in\mathbb{R}^{d}\}$ in the same layer:
\begin{equation}\label{energy}
\footnotesize
    \min_{\{\hat{\bm{w}}_1,\cdots,\hat{\bm{w}}_n\in\mathbb{S}^{d-1}\}}\big{\{} E_s(\hat{\bm{W}_n}):=\sum_{i=1}^{n}\sum_{j=1,j\neq i}^{n}
    K_s(\hat{\bm{w}}_i,\hat{\bm{w}}_j)\big{\}}
    %\frac{1}{\rho(\hat{\bm{w}}_i,\hat{\bm{w}}_j)^s}\big{\}}
\end{equation}
where $\hat{\bm{w}}_i:=\frac{\bm{w}_i}{\|\bm{w}_i\|}$ is the $i$-th vector projected onto the unit hypersphere $\mathbb{S}^{d-1}=\{\hat{\bm{w}}\in\mathbb{R}^{d}|\norm{\hat{\bm{w}}}=1\}$. $K_s(\cdot,\cdot)$ models the interaction between two vectors, and we will typically consider the following Riesz $s$-kernel function: 
\begin{equation}
\footnotesize
   K_s(\hat{\bm{w}}_i,\hat{\bm{w}}_j) =\left\{
{\begin{array}{*{20}{l}}
{\rho(\hat{\bm{w}}_i,\hat{\bm{w}}_j)^{-s},\ \ \ s>0}\\
{\log(\rho(\hat{\bm{w}}_i,\hat{\bm{w}}_j)^{-1}),\ \ \ s=0}\\
{-\rho(\hat{\bm{w}}_i,\hat{\bm{w}}_j)^{-s},\ \ \ s<0}
\end{array}} \right.
\end{equation}
where $\rho(\cdot,\cdot)$ is defined to measure the geodesic similarity on the unit hypersphere. In general, we can use either  $\rho(\hat{\bm{w}}_i,\hat{\bm{w}}_j)=\|\hat{\bm{w}}_i-\hat{\bm{w}}_j\|_2$ (\ie, standard Riesz $s$-kernel) or hyperspherical geodesic distance (\ie, angular distance) $\rho(\hat{\bm{w}}_i,\hat{\bm{w}}_j)=\arccos(\hat{\bm{w}}_i^\top\hat{\bm{w}}_j)$. Moreover, it is  known that the minimizer of this pairwise energy sum asymptotically corresponds to the uniform distribution on the hypersphere~\cite{liu2018learning,kuijlaars1998asymptotics,borodachov2019discrete}. Since the objective of MHE is non-convex and differentiable, we typically optimize it with gradient descent. Different $s$ usually yields slightly different regularization effects and optimization dynamics. MHE has been extensively studied and shown to be effective in many applications~\cite{liu2018learning}.

\vspace{-2mm}
\subsection{Maximum Hyperspherical Separation}
\vspace{-1.5mm}

MHS adopts a maximum geodesic separation criterion by maximizing the \emph{separation distance} (\ie, the smallest pairwise distance in a vector set, denoted as $\vartheta(\hat{\bm{W}}_n)$) which is equal to the smallest geodesic distance between any two vectors in the set $\hat{\bm{W}}_n=\{\hat{\bm{w}}_1,\cdots,\hat{\bm{w}}_n\in\mathbb{S}^{d-1}\}$:
\begin{equation}\label{sep_dis}
\footnotesize
    \max_{\{\hat{\bm{w}}_1,\cdots,\hat{\bm{w}}_n\in\mathbb{S}^{d-1}\}}\big{\{}\vartheta(\hat{\bm{W}}_n):=\min_{i\neq j} \rho(\hat{\bm{w}}_i,\hat{\bm{w}}_j)\big{\}}
\end{equation}
which is formulated as a max-min optimization problem. It is originally called Tammes problem~\cite{tammes1930origin} (or sphere packing problem) where one packs a given number of circles on the surface of a sphere such that the minimum distance between circles can be maximized. Similar to MHE, $\rho(\cdot,\cdot)$ can be either Euclidean distance or geodesic distance on the unit hypersphere. Based on the following proposition, we can obtain that MHS is in fact a limiting case of MHE when $s\rightarrow\infty$.

\vspace{1.4mm}
\begin{proposition}\label{mhs}
Let $n\in\mathbb{N},n\geq2$ be fixed and $(\mathbb{S}^{d-1},\rho)$ be a compact metric space. Then we have that
\begin{equation}
\footnotesize
    \lim_{s\rightarrow\infty}\big(\varepsilon_s(\mathbb{S}^{d-1},n)\big)^{\frac{1}{s}}=\frac{1}{\delta_n^{\rho}(\mathbb{S}^{d-1})}
\end{equation}
where we define $\varepsilon_s(\mathbb{S}^{d-1},n):=\min_{\hat{\bm{W}}_n\subset\mathbb{S}^{d-1}}E_s(\hat{\bm{W}}_n)$ and $\delta^{\rho}_n(\mathbb{S}^{d-1}):=\max_{\hat{\bm{W}}_n\subset\mathbb{S}^{d-1}}\vartheta(\hat{\bm{W}}_n)$.
\end{proposition}

Compared to MHE that has a global regularization effect, MHS focuses more on the local separation since it only takes the minimal geodesic distance into consideration. Specifically, MHS only updates two vectors with smallest angular distance at each iteration, while MHE updates all vectors in each iteration. When the vectors are relatively diverse on the hypersphere, MHS tends to have stronger regularization effects than MHE. This is because the MHS gradient resulted from the two closest vectors has only one direction component and will not be cancelled out. In contrast, the MHE gradient that comes from all the vectors has many direction components and may be largely cancelled out. More intuitively, MHS only activates the repulsive force from the closest vectors at a time while MHE simultaneously activates the repulsive force from all pairwise vectors.

Optimizing MHS is straightforward and efficient. We first need to use a ranking operator to rank all the pairwise distances and obtain the vectors with the smallest distance (\ie, maximal similarity). Then we can simply maximize the minimal distance by updating these two closest vectors via gradient ascent.

\vspace{-1.9mm}
\subsection{Maximum Hyperspherical Polarization}
\vspace{-1.4mm}

MHP arises from a practical problem: if $K_s(\bm{v},\hat{\bm{w}}_i)$ denotes the amount of a substance received at $\bm{v}$ due to an injector of the substance located at $\hat{\bm{w}}_i$, what is the smallest number of injectors and their corresponding optimal locations on the hypersphere so that a prescribed minimal amount of the substance can reach every point on the hypersphere? Specifically, MHP maximizes the following $s$-polarization $P_s(\hat{\bm{W}}_n)$ of a $n$-vector set $\hat{\bm{W}}_n=\{\hat{\bm{w}}_1,\cdots,\hat{\bm{w}}_n\in\mathbb{S}^{d-1}\}$:
\begin{equation}
\footnotesize
    \max_{\{\hat{\bm{w}}_1,\cdots,\hat{\bm{w}}_n\in\mathbb{S}^{d-1}\}}\big{\{}P_s(\hat{\bm{W}}_n):=\min_{\bm{v}\in\mathbb{S}^{d-1}}\sum_{i=1}^nK_s(\bm{v},\hat{\bm{w}}_i)\big{\}}
\end{equation}
which is a max-min problem and amounts to identifying the optimal location of ``poles'' for the potential function. Similar to MHE, the potential can be modeled by Riesz $s$-kernel (\eg, $K(\bm{v},\hat{\bm{w}}_i)=\rho(\bm{v},\hat{\bm{w}}_i)^{-s}$). We show an intrinsic relationship between MHE and MHP:

\vspace{1.4mm}
\begin{proposition}\label{mhp}
For every $n\in\mathbb{N}, n\geq 2$ and a compact metric space $(\mathbb{S}^{d-1},\rho)$, we have that
\begin{equation}
\footnotesize
    \mathcal{P}_s(\mathbb{S}^{d-1},n)\geq\frac{\varepsilon_s(\mathbb{S}^{d-1},n+1)}{n+1}\geq\frac{\varepsilon_s(\mathbb{S}^{d-1},n)}{n-1}
\end{equation}
where $\mathcal{P}_s(\mathbb{S}^{d-1},n):=\max_{\hat{\bm{W}}_n\subset\mathbb{S}^{d-1}}P_s(\hat{\bm{W}}_n)$.
\end{proposition}

Next, we show an interesting example where the maximal $s$-polarization problem can be easily solved.

\vspace{1.4mm}
\begin{proposition}\label{mhp_relax}
For the case of $s=-2$ and $n\geq 2$, a vector configuration on $\mathbb{S}^{d-1}$: $\hat{\bm{W}}=\{\hat{\bm{w}}_1,\cdots,\hat{\bm{w}}_n\}$ is an optimal solution for the maximal $s$-polarization problem if and only if $\sum_{i=1}^n\hat{\bm{w}}_i=\bm{0}$.
\end{proposition}

Proposition~\ref{mhp_relax} inspires us to propose a relaxed MHP (R-MHP) regularization objective:  $\min_{\hat{\bm{W}}_n}\norm{\sum_{i=1}^n\hat{\bm{w}}_i}$. Although it only corresponds to a specific maximal $(-2)$-polarization solution, R-MHP can still serve as an interesting relaxed variant of the original MHP regularization for neurons. Geometrically, R-MHP can be viewed as constraining the mass center of all the points to the origin. Note that, R-MHP is a necessary condition to achieve hyperspherical uniformity, and therefore is unable to guarantee hyperspherical uniformity. Moreover, there exists a trivial solution for R-MHP to achieve the optimum, where every two vectors are paired to have the opposite directions. Typically the data fitting loss in training can prevent the neurons from falling into such a trivial solution. Alternatively we can make the number of neurons to be odd to eliminate the existence of such a trivial solution, since the neurons can no longer be paired in this case.

To optimize the standard MHP, we propose to unroll the inner minimization with a few gradient descent steps and then embed it back to the outer maximization, which is conceptually similar to \cite{dai2018coupled,finn2017model,monga2019algorithm,andrychowicz2016learning}.

\vspace{-2mm}
\subsection{Minimum Hyperspherical Covering}
\vspace{-1.5mm}
MHC minimizes the following \emph{covering radius} $\alpha(\hat{\bm{W}}_n)$ (also known as \emph{mesh norm}) of a $n$-vector set $\hat{\bm{W}}_n$:
\begin{equation}\label{eq_mhc}
\footnotesize
    \min_{\{\hat{\bm{w}}_1,\cdots,\hat{\bm{w}}_n\in\mathbb{S}^{d-1}\}}\big{\{}\alpha(\hat{\bm{W}}_n):=\max_{\bm{v}\in\mathbb{S}^{d-1}}\min_{1\geq i\geq n}\rho(\bm{v},\hat{\bm{w}}_i)\big{\}}
\end{equation}
where the covering radius $\alpha(\hat{\bm{W}}_n)$ denotes the maximum geodesic distance from a vector in $\mathbb{S}^{d-1}$ to the nearest vector in $\hat{\bm{W}}_n$, and can also be viewed as the geodesic radius of the largest hyperspherical cap that contains no points from $\hat{\bm{W}}_n$. Covering radius has interesting applications in numerical integration on the sphere~\cite{sloan2004extremal}. We further show that MHC is the limiting case of the maximal $s$-polarization (\ie, MHP) as $s\rightarrow\infty$.

\vspace{1.5mm}
\begin{proposition}\label{mhc}
For a compact metric space $(\mathbb{S}^{d-1},\rho)$, we can have the following equation:
\begin{equation}
\footnotesize
    \lim_{s\rightarrow\infty}\big(\mathcal{P}_s(\mathbb{S}^{d-1},n)\big)^{\frac{1}{s}}=\frac{1}{\eta_n^\rho(\mathbb{S}^{d-1})}
\end{equation}
where $\eta_n^\rho(\mathbb{S}^{d-1}):=\min_{\hat{\bm{W}}_n\subset\mathbb{S}^{d-1}}\alpha(\hat{\bm{W}}_n)$ is the minimal $n$-point covering radius in $\mathbb{S}^{d-1}$.
\end{proposition}

Proposition~\ref{mhc} shows that MHC is a special case of MHP. However, it is highly nontrivial to effectively optimize the MHC objective since it involves two inner optimizations. One simple and straightforward way is to alternatively optimize this min-max-min problem. Alternatively, we also propose to use a smooth minimum operator is to approximate and relax the most inner minimization in order to make it differentiable, such that we can then use the unrolling technique to solve the relaxed optimization. Specifically, we can relax Eq.~\ref{eq_mhc} to the following simpler problem:
\begin{equation}\label{eq_mhc2}
\footnotesize
    \min_{\{\hat{\bm{w}}_1,\cdots,\hat{\bm{w}}_n\in\mathbb{S}^{d-1}\}}\max_{\bm{v}\in\mathbb{S}^{d-1}}-\frac{1}{\gamma}\log \sum_{i=1}^n \exp\big({-\gamma \rho(\bm{v},\hat{\bm{w}}_i)}\big)
\end{equation}
whose inner maximization can be unrolled with a few gradient ascent steps, similar to MHP. Larger $\gamma$ gives better approximation to the minimum operator.

\vspace{-2mm}
\subsection{Maximum Gram Determinant}
\vspace{-1.5mm}
Inspired by numerical integration~\cite{sloan2004extremal} and interpolation~\cite{karvonen2020kernel}, hyperspherical uniformity can be achieved by maximizing the determinant of the kernel Gram matrix of the $n$-vector set $\hat{\bm{W}}_n=\{\hat{\bm{w}}_1,\cdots,\hat{\bm{w}}_n\in\mathbb{S}^{d-1}\}$:
\begin{equation}
\footnotesize
    \max_{\{\hat{\bm{w}}_1,\cdots,\hat{\bm{w}}_n\in\mathbb{S}^{d-1}\}}\log\det\bigg(\bm{G}:=\big(K(\hat{\bm{w}}_i,\hat{\bm{w}}_j)\big)_{i,j=1}^n\bigg)
\end{equation}
where $\det(\bm{G})$ denotes the determinant of the kernel gram matrix $\bm{G}\in\mathbb{R}^{n\times n}$, and $K(\bm{u},\bm{v})$ can be expanded with $\sum_{i=1}^\infty\psi_i(\bm{u})\psi_i(\bm{v})$. $\{\psi_i\}_{i=1}^\infty$ is an orthogonal basis of $\mathcal{H}_K(\mathbb{S}^{d-1})$, the reproducing kernel Hilbert space of $K$. By approximating the kernel with the first $n$ terms, then we can have that $\tilde{\bm{G}}=\bm{\Psi}_n\bm{\Psi}_n^\top\approx\bm{G}$ in which
\begin{equation}
\footnotesize
    \bm{\Psi}_n=\begin{bmatrix}
    \psi_1(\hat{\bm{w}}_1)&\cdots&\psi_n(\hat{\bm{w}}_1)\\
    \vdots & \ddots & \vdots\\
    \psi_1(\hat{\bm{w}}_n)&\cdots&\psi_n(\hat{\bm{w}}_n)\\
    \end{bmatrix}.
\end{equation}
Therefore, maximizing $\det(\bm{G})$ can be approximated by maximizing $|\det(\bm{\Psi_n})|$ which is also known as \emph{extremal systems}~\cite{sloan2004extremal} or \emph{Fekete points}~\cite{marzo2010equidistribution,berman2011fekete,karvonen2020kernel}. Note that, the construction of Fekete points is independent of the choice of basis. We can therefore use Gaussian kernel:
\begin{equation}
\footnotesize
    K(\bm{u},\bm{v})=\exp\big( -\sum_{i=1}^d\epsilon^2(u_i-v_i)^2 \big)
\end{equation}
where $\epsilon>0$ is a scale parameter. It is also possible to use some other kernel functions, but we stick to the Gaussian kernel for simplicity. Geometrically, the kernel Gram determinant is also closely related to $n$-dimensional volume of the parallelotope formed by $\hat{\bm{W}}_n$.

\vspace{-2mm}
\subsection{Theoretical Results on Optimality}\label{optimality}
\vspace{-1.5mm}

Computing the optimal solutions to these optimizations is highly challenging and often infeasible~\cite{borodachov2019discrete}. Fortunately, under conditions on the dimensionality and the number of vectors, we can characterize the optimum.

\vspace{1.4mm}
\begin{theorem}[Simplex optimum for $2\leq n\leq d+2$]\label{optimum_1}
Let $f:(0,4]\rightarrow \mathbb{R}$ be a convex and decreasing function defined at $t=0$ by $\lim_{v\rightarrow0^+}f(v)$. If $2\leq n\leq d+2$, then we have that the vertices of regular $(n-1)$-simplices inscribed in $\mathbb{S}^d$ with centers at the origin minimize the MHE objective on the hypersphere $\mathbb{S}^d$ ($d\geq2$) with the kernel as $K_s(\hat{\bm{w}}_i,\hat{\bm{w}}_j)=f(\|\hat{\bm{w}}_i-\hat{\bm{w}}_j\|^2)$. Moreover, $f$ is strictly convex and strictly decreasing, then these are the only energy minimizing $n$-point configurations.
\end{theorem}

Theorem~\ref{optimum_1} indicates that the vertices of a regular $(d+1)$-simplex (\ie, $(d+1)$-dimensional convex hull of $d+2$ distinct vectors with equal pairwise distances) are universally optimal. Universal optimality~\cite{cohn2007universally} refers to a finite subset $\mathcal{C}\subset\mathbb{S}^d$ that minimizes potential energy $\sum_{\{\bm{x}\neq\bm{y}\}\in \mathbb{S}^d} f(\|\bm{x}-\bm{y}\|^2)$ among all configurations of $|\mathcal{C}|$ points on $\mathbb{S}^d$ for every completely monotonic potential function $f$. Particularly, the vertices of a regular $(d+1)$-simplex are the minimizer of MHE with the Riesz $s$-kernel (for $s\geq-2$), MHS, and MHC.

\vspace{1.4mm}
\begin{theorem}[Cross-polytope optimum for $n= 2d+2$]\label{optimum_2}
The vertices of the $\frac{n}{2}$-cross-polytope form a universally optimal $n$-point configuration on $\mathbb{S}^d$.
\end{theorem}

Theorem~\ref{optimum_2} can be obtained from \cite{borodachov2019discrete,cohn2007universally,kolushov1997extremal,yudin1992minimum} and it indicates the cross-polytope point configurations are the minimizer of MHE and MHS. A cross-polytope can be constructed by the convex hull of unit vectors pointing along each Cartesian coordinate axis. For example, the convex hull of 6 points: $(0,0,\pm1)$, $(0,\pm1,0)$, and $(\pm1,0,0)$ is a $3$-cross-polytope. Most importantly, this result builds a bridge that intrinsically connects hyperspherical uniformity and orthogonality. Orthogonality works in a relatively restricted setting and can be viewed as a special case of hyperspherical uniformity.

It is easy to verify that both simplex and cross-polytope form optimal point configurations in $\mathbb{S}^1$ and $\mathbb{S}^2$. Theorem~\ref{optimum_1} and Theorem~\ref{optimum_2} essentially show that such a conclusion generalizes intuitively to high dimensional space. Moreover, \cite{cohn2007universally} shows that there is no universally optimal point set on $\mathbb{S}^d$ that contains more points than the regular simplex and fewer points than the cross-polytope (\ie, $d+2<n<2d+2$). For $n>2d+2$, characterizing optimum becomes much more involved~\cite{borodachov2019discrete}.

\vspace{-2.2mm}
\section{A Statistical Perspective from Uniformity Testing on Hypersphere}
\vspace{-1.7mm}

This section aims to gain more interesting insights by casting an alternative view on hyperspherical uniformity from a statistical uniformity testing perspective. Since we aim to promote hyperspherical uniformity, we can draw inspirations from the formulation of the test statistic in uniformity testing. Given \emph{i.i.d.} samples $\bm{u}_1,\cdots,\bm{u}_n\in\mathbb{S}^{d-1}$ of a unit random vector $\bm{u}$, the assessment of the presence of uniformity on the hypersphere is formalized as the testing of the null hypothesis $\mathcal{H}_0:\bm{P}=\textnormal{Uniform}(\mathbb{S}^{d-1})$ against $\mathcal{H}_1:\bm{P}\neq\textnormal{Uniform}(\mathbb{S}^{d-1})$, where $\bm{P}$ denotes the probability distribution of $\bm{u}$ and $\textnormal{Uniform}(\mathbb{S}^{d-1})$ denotes the uniform distribution on $\mathbb{S}^{d-1}$. 
\par
The intuitive idea of the Sobolev test is to map the hypersphere $\mathbb{S}^{d-1}$ into the Hilbert space $L^2(\mathbb{S}^{d-1},\mu)$ of square-integrable functions on $\mathbb{S}^{d-1}$ by a function $t:\mathbb{S}^{d-1}\rightarrow L^2(\mathbb{S}^{d-1},\mu)$ such that, if $\bm{u}$ is uniformly distributed on the hypersphere, then the mean of $t(\bm{u})$ will be zero. Specifically, we denote $\bm{\epsilon}_k$ ($p_{d,k}=\textnormal{dim}\bm{\epsilon}_k$) as the space of eigenfunctions $\mathbb{S}^{d-1}\rightarrow\mathbb{R}$ corresponding to the $k$-th non-zero eigenvalue of the Laplacian, there exists a well-defined mapping $t_k:\mathbb{S}^{d-1}\rightarrow\bm{\epsilon}_k$ which can be written as $t_k(\bm{u}):=\sum_{i=1}^{p_{d,k}}g_{i,k}(\bm{u})g_{i,k}$ ($\{g_{i,k}\}_{i=1}^{p_{d,k}}$ constructs an orthonormal basis of $\bm{\epsilon}_k$). We let $\{v_k\}_{k=1}^\infty$ be a sequence such that $\sum_{k=1}^\infty v_k^2 p_{d,k}<\infty$, and then the function $\bm{u}\rightarrow t(\bm{u}):=\sum_{k=1}^\infty\bm{v}_k t_k(\bm{u})$ is a mapping from $\mathbb{S}^{d-1}$ to the Hilbert space $L^2(\mathbb{S}^{d-1},\mu)$ of square-integrable real functions on $\mathbb{S}^{d-1}$ \emph{w.r.t.} $\mu$ which is the uniform measure on $\mathbb{S}^{d-1}$. The Sobolev test~\cite{beran1968testing,gine1975invariant} rejects $\mathcal{H}_0$ for large values of the following test statistic:

\vspace{-3.2mm}
\begin{equation}
\footnotesize
    S_n:=\frac{1}{n}\sum_{i,j}^n\sum_{k=1}^\infty v_k^2\langle t_k(\bm{u}_i),t_k(\bm{u}_j)\rangle
\end{equation}
\vspace{-3.2mm}

where $\langle f,g \rangle:=\int_{\mathbb{S}^{d-1}}f(\bm{u})g(\bm{u})\textnormal{d}\mu(\bm{u})$ denotes the inner product on $L^2(\mathbb{S}^{d-1},\mu)$. \cite{prentice1978invariant} gives an explicit form for $\langle t_k(\bm{u}_i),t_k(\bm{u}_j)\rangle$ in $\mathbb{S}^{d-1}$ (assuming $d>2$):

\vspace{-5.5mm}
\begin{equation}
\footnotesize
\langle t_k(\bm{u}_i),t_k(\bm{u}_j)\rangle=(1+\frac{2k}{d-2})C_k^{(d-2)/2}(\bm{u}_i^\top\bm{u}_j)
\end{equation}
\vspace{-5mm}

where $C_k^\alpha(\cdot)$ denotes the Gegenbauer polynomial of index $\alpha$ and order $k$. The  asymptotic distribution of $S_n$ under $\mathcal{H}_0$ is the infinite linear combination of independent chi-squared distribution $\sum_{k=1}^\infty v_k^2\mathcal{X}_{p_{d,k}}^2$.

\textbf{Connection to MHE}. Let $v_k=0$ when $k$ is even and $v_k=(\pi k)^{-1}$ when $k$ is odd. As one of the Sobolev test, Ajne test~\cite{prentice1978invariant,ajne1968simple} uses the following test statistic: 

\vspace{-3.5mm}
\begin{equation}\label{ajne}
\footnotesize
    A_n=\frac{n}{4}-\frac{1}{n\pi}\sum_{1\leq i <j\leq n}\arccos(\bm{u}_i^\top\bm{u}_j)
\end{equation}
\vspace{-3.5mm}

which rejects $\mathcal{H}_0$ for large values. It is well connected to MHE in the sense that both MHE and Eq.~\eqref{ajne} are formulated based on pairwise relations. More specifically, Minimizing Eq.~\eqref{ajne} \emph{w.r.t.} $\{\bm{u}_i\}_{i=1}^n$ is in fact equivalent to MHE with Riesz $s$-kernel where $s=-1$.
\par

\textbf{Connection to R-MHP}. Rayleigh test~\cite{rayleigh1919xxxi}, which is a special case of the Sobolev test (with $v_1=1$ and $v_k=0$ for $k\geq2$), gives the test statistic in $\mathbb{S}^{d-1}$: $R_n=nd\|\bar{\bm{u}}\|^2$ where $\bar{\bm{u}}:=\frac{1}{n}\sum_{i=1}^n\bm{u}_i$. $R_n$ is asymptotically distributed as $\mathcal{X}_d^2$ under $\mathcal{H}_0$. Therefore, $\|\bar{\bm{u}}\|$ is approaching to zero, which exactly matches the objective function of the R-MHP.
\par

\textbf{Connection to MHS and MHC}. Spacing tests on $\mathbb{S}^1$ are constructed from the gaps between the ordered samples (\ie, $\thickmuskip=2mu \medmuskip=2mu \theta_{i+1}>\theta_{i},\forall i$): $d_i:=\theta_{i+1}-\theta_{i}$ where $i=1,\cdots,n-1$ and $d_n:=2\pi-(\theta_n-\theta_{1})$. A special spacing test, called Range test~\cite{rao1969some}, has the following statistic: $T_n:=2\pi-\max_{i} d_i$ which rejects $\mathcal{H}_0$ with low values. Then we have the following proposition that builds connection among range test, MHS and MHC:

\vspace{1.4mm}
\begin{proposition}\label{range_test}
Maximizing the statistic of range test on $\mathbb{S}^1$ is equivalent to MHS and MHC on $\mathbb{S}^1$.
\end{proposition}

In summary, there is a close connection between statistical uniformity testing on the hypersphere and our proposed objectives towards hyperspherical uniformity. Revisiting classic uniformity tests can not only help us to gain more insights, but also inspire more useful ways to promote hyperspherical uniformity.

\vspace{-2.3mm}
\section{Discussions}
\vspace{-2.05mm}

\textbf{Removing collinearity}. Promoting hyperspherical uniformity may lead to the neuron collinearity problem, which causes redundancy and is therefore undesirable in neural networks~\cite{roychowdhury2018reducing,shang2016understanding}. In order to address this, we use the virtual neuron trick in \cite{liu2018learning,lyu2017spherical}. We construct a set of virtual neurons that always have the opposite directions to the original neurons and then regularize both original and virtual neurons together. Specifically, we assume a set of original neurons $\{\bm{w}_i\}_{i=1}^n$ and the virtual opposite neurons are $\{\bm{v}_i=-\bm{w}_i\}_{i=1}^n$. Finally we will apply hyperspherical uniformity to simultaneously regularize all the neurons $\{\bm{w}_1,\cdots,\bm{w}_n,\bm{v}_1,\cdots,\bm{v}_n\}$. The virtual neuron trick is not applicable to R-MHP.

\vspace{-0.4mm}

\textbf{Why hyperspherical uniformity}. A popular choice out of many existing relational regularizations is the orthogonality regularization~\cite{brock2016neural,liu2017hyper,bansal2018can,huang2018orthogonal,jia2019orthogonal,qi2020deep}. Despite its popularity, orthogonality yields a few drawbacks. When the number of neurons exceeds the neuron dimension, promoting orthogonality among neurons will become problematic~\cite{liu2018learning}. In contrast, hyperspherical uniformity avoids such a problem while also being more general. It works well in all circumstances.

\vspace{-0.4mm}

\textbf{Increasing effective width of neural networks}. Empirical evidences in \cite{roychowdhury2018reducing,shang2016understanding,han2015deep,han2015learning} show that naively training a neural network typically leads to severe neuron redundancy and \cite{zagoruyko2016wide} shows that increasing the width of neural networks can significantly improve generalization. Hyperspherical uniformity can increase the effective width and improve the representation efficiency by penalizing hyperspherically similar neurons.

\vspace{-0.4mm}

\textbf{Decoupling neuron norm and angle}. The norm of neurons is typically regularized by individual regularizations, while the relational regularizations are designed for neuron directions. Orthonormality does not decouple the two components, since it also regularizes neuron norm to be close to $1$, which serves a redundant role to weight decay. In contrast, hyperspherical uniformity fully decouples the two regularization components.

\vspace{-0.38mm}

\textbf{Optimization difficulty}. Although all the proposed regularization objectives share the same goal, there is a substantial difference when actually performing gradient descent with them. MHP and MHC define max-min (or min-max) problems which require alternative update or unrolled formulation to solve, so optimizing them may lead to much more bad local minima and therefore is not as stable as MHE, MHS and MGD.

\begin{figure}[t]
\vspace{0.8mm}
  \centering
  \renewcommand{\captionlabelfont}{\scriptsize}
  \setlength{\abovecaptionskip}{1.6pt}
  \setlength{\belowcaptionskip}{-11pt}
\includegraphics[width=3.25in]{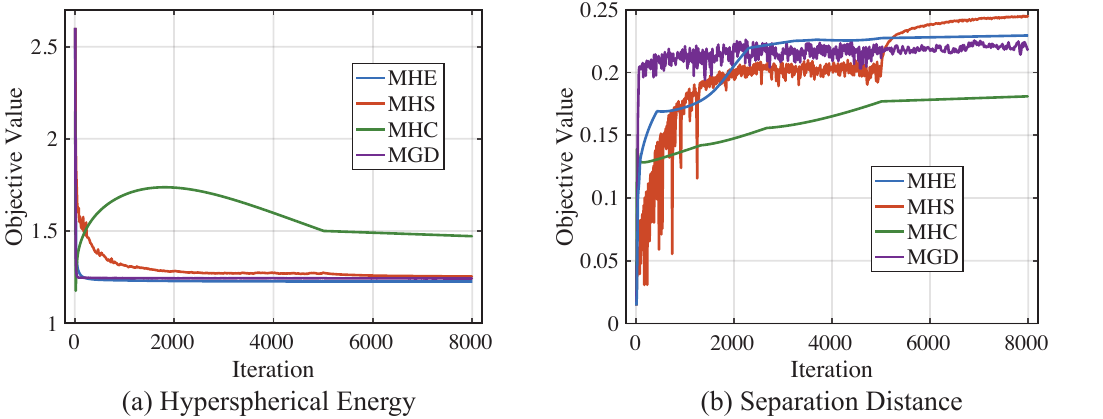}
  \caption{\scriptsize Comparison of local separatedness and global uniformity among MHE, MHS, MHC and MGD. (a) Hyperspherical energy vs. Iteration. (b) Separation distance vs. Iteration. Experimental details are in Appendix~\ref{exp_detail}. The figure is best viewed in color.}\label{separation_exp}
\end{figure}

\vspace{-0.3mm}
\section{Properties and Insights}\label{theoretical_insights}

\vspace{-2.4mm}
\subsection{Geometric Properties}
\vspace{-2mm}
We discuss the geometric properties of the proposed regularizations by first connecting them with the hyperspherical uniformity. \cite{liu2018learning} has established that minimizing MHE ultimately leads to hyperspherical uniformity. This is achieved by showing that unit point mass at each vector in $\hat{\bm{W}}_n$ asymptotically approaches to the spherical measure on $\mathbb{S}^{d-1}$. \cite{kolmogorov1959varepsilon,borodachov2007asymptotics} have shown that asymptotically best-packing points on a rectifiable set are uniformly distributed. MHS is essentially a special case of this result, indicating that MHS also asymptotically leads to hyperspherical uniformity. \cite{borodachov2014asymptotics,borodachov2018optimal,erdelyi2013riesz} and \cite{kolmogorov1959varepsilon} prove the same argument holds for MHP and MHC, respectively. \cite{marzo2010equidistribution} proves the asymptotic equidistribution of Fekete points on the hypersphere, implying that MGD also leads to hyperspherical uniformity.
\par
Although all the regularization objectives asymptotically leads to hyperspherical uniformity, there are still some delicate differences in terms of the geometric properties. For example, MHS focuses on more on \emph{local separatedness} by ensuring any nearest two vectors to be far away, while MHE characterizes the \emph{global uniformity} with the sum of pairwise energies. Therefore, we can use the separation distance (\ie, $\vartheta(\hat{\bm{W}}_n)$ in Eq.~\eqref{sep_dis}) as a measure of local separatedness and use the hyperspherical energy (\ie, $E_s(\hat{\bm{W}}_n)$ in Eq.~\eqref{energy}) as a measure of global uniformity. We connect MHE and MHC with the following geometric properties.

\begin{figure*}[t]
  \centering
  \setlength{\abovecaptionskip}{4pt}
  \setlength{\belowcaptionskip}{-6pt}
  \renewcommand{\captionlabelfont}{\scriptsize}
    \includegraphics[width=1\textwidth]{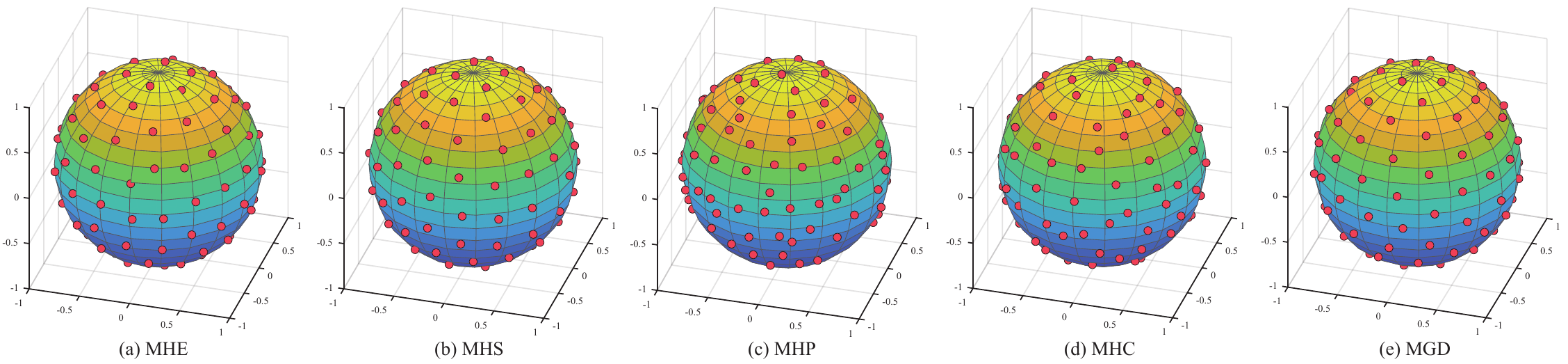}
    \caption{\scriptsize Visualization of regularization effects for MHE, MHS, MHP, MHC and MGD. We naively optimize all the proposed uniformity regularizations with gradient descent on 3D spherical data. Experimental details are given in Appendix~\ref{exp_detail}. The figure is best viewed in color.}\label{3d_sphere}
\end{figure*}

\begin{theorem}\label{lower_MHE}
For $d-2\leq s< d-1$, there is a constant $\lambda_{s,d}>0$ such that for $n\geq 2$ and any solution $\hat{\bm{W}}^*_n$ that attains the optima of MHE, we have $\vartheta(\hat{\bm{W}}^*_n)\geq\frac{\lambda_{s,d}}{n^{1/(d-1)}}$.
\end{theorem}

\vspace{-1.05mm}

\begin{theorem}\label{energy_MHS}
When $0<s<d-1$ and $n$ approaches infinity, every optimal MHS solution on $\mathbb{S}^{d-1}$ is asymptotically an optimal solution for MHE.
\end{theorem}

\vspace{-1.65mm}

Theorem~\ref{lower_MHE} gives a lower bound of separation distance for any MHE solution. It further indicates that optimal MHE configurations are well-separated and minimizing MHE objective can partially maximize the MHS objective. Theorem~\ref{energy_MHS} shows that optimal MHS configurations will asymptotically have the same global uniformity (\ie, hyperspherical energy) as optimal MHE configurations. These two results well connect the geometric properties between the MHE and MHS.

\vspace{-0.5mm}

In fact, it requires involved analyses to fully study the geometric properties of these objectives~\cite{borodachov2007asymptotics,borodachov2019discrete}, which is out of our scope. In order to gain more insights, we conduct a simulation in Fig.~\ref{separation_exp} to empirically compare MHE, MHS, MHC and MGD. Specifically, we initialize 200 neurons and separately optimize MHE, MHS, MHC or MGD using gradient descent with these neurons. We compute the separation distance for all regularizations during training to compare their local separatedness, and also compute the hyperspherical energy to compare the global uniformity. We see that MHS and MHE achieve the best local separatedness and global uniformity, respectively. This is because MHS and MHE are directly optimizing the measures. MHE, MHS and MGD have similar global uniformity since they achieve similar hyperspherical energy. MHE, MGD and MHS reach a good balance between global uniformity and local separatedness, justifying their effectiveness. In contrast, MHC has worse geometric properties due to the complex objective. The results verify that although these hyperspherical uniformity regularizations share the same goal, the encoded regularization effects and inductive biases can be vastly different.

\vspace{-1.65mm}
\subsection{Spectral Properties}
\vspace{-1.45mm}

While orthogonality guarantees that all singular values are $1$, the spectral properties of hyperspherical uniformity is discussed in the following theorem.
\vspace{1.25mm}

\begin{theorem}\label{spectral}
Let $\thickmuskip=2mu \medmuskip=2mu \tilde{\bm{v}}_1,\cdots,\tilde{\bm{v}}_n$ be i.i.d. random vectors where each element follows the Gaussian distribution with mean $0$ and variance $1$. Then $\thickmuskip=2mu \medmuskip=2mu \bm{v}_1=\frac{\tilde{\bm{v}}_1}{\|\tilde{\bm{v}}_1\|},\cdots,\bm{v}_n=\frac{\tilde{\bm{v}}_1}{\|\tilde{\bm{v}}_1\|}$ are uniformly distributed on the unit hypersphere $\mathbb{S}^{d-1}$. If the ratio $\frac{n}{d}$ converges to a constant $\thickmuskip=2mu \medmuskip=2mu \lambda\in(0,1)$, asymptotically we have for $\thickmuskip=2mu \medmuskip=2mu \bm{W}=\{\bm{v}_1,\cdots,\bm{v}_n\}\in\mathbb{R}^{d\times n}$:
\begin{equation}\label{spectral_gap}
\footnotesize
\begin{aligned}
    \lim_{n\rightarrow\infty}\sigma_{\max}(\bm{W})&\leq(\sqrt{d}+\sqrt{\lambda d})\cdot(\max_i\frac{1}{\|\tilde{\bm{v}}_i\|_2})\\
    \lim_{n\rightarrow\infty}\sigma_{\min}(\bm{W})&\geq(\sqrt{d}-\sqrt{\lambda d})\cdot(\min_i\frac{1}{\|\tilde{\bm{v}}_i\|_2})
\end{aligned}
\end{equation}
where $\sigma_{\max}(\cdot)$ and $\sigma_{\min}(\cdot)$ denote the largest and the smallest singular value of a matrix, respectively.
\end{theorem}

\vspace{-0.5mm}

Theorem~\ref{spectral} is the direct application of a classic result on the limits of the largest and smallest singular values of Gaussian matrix~\cite{silverstein1985smallest}. Theorem~\ref{spectral} guarantees that hyperspherical uniformity will constrain the largest singular value of a matrix from being too large and its smallest singular value from being too small.

\vspace{-2mm}
\subsection{Connections to Neuron Initialization}\label{neuron_init}
\vspace{-2mm}

We also discuss how hyperspherical uniformity is connected to the neuron initialization and its potential connection to Occam's razor. First, almost all initialization schemes~\cite{lecun2012efficient,he2015delving,glorot2010understanding} for neural networks initializes the directions of neurons in the same layer to be uniformly distributed on the hypersphere, which is proved in Appendix~\ref{uniform_sphere}. Other than the zero-mean Gaussian distribution used in \cite{lecun2012efficient,he2015delving,glorot2010understanding}, we may alternatively characterize the hyperspherical uniformity with 
Cauchy-like distribution~\cite{szablowski1998uniform}, which may inspire alternative initializations. 
Promoting hyperspherical uniformity in training aims to regularize the parameters to be similar to the initialization in terms of the directional support (since the pairwise relationship among neurons on the hypersphere is encouraged to be close to the initialization), which is essentially making the neural network to be as ``simple'' as possible. In contrast to the weight decay that implements Occam's razor in terms of magnitude (\ie, discrepancy to $0$), promoting hyperspherical uniformity can essentially be viewed as implementing Occam's razor in the angular space (\ie, discrepancy to hyperspherical uniformity).

\vspace{-2mm}
\subsection{Regularization Effects}
\vspace{-2mm}

In order to intuitively understand these different objectives towards hyperspherical uniformity, we visualize their regularization effects by minimizing them with gradient descent on a 3D sphere. The results in Fig.~\ref{3d_sphere} show that MHE, MHS and MGD yield visually superior regularization effects for hyperspherical uniformity, while MHP and MHC give suboptimal and weaker regularization effects partially due to their complex optimization objectives. Because both MHP and MHC involve an inner optimization problem, their objective landscapes are highly non-convex and also more difficult to optimize. In contrast, MHE, MHS and MGD have much strong regularization effects on the 3-sphere.

\vspace{-2.3mm}
\section{Applications}
\vspace{-2mm}

Our goal is to show the performance gain by applying hyperspherical uniformity to different applications instead of achieving state-of-the-art performance. For fairness, we always stick to clean baselines without bells and whistles and ensure the experimental setup is the same for all compared methods. All the detailed experimental settings are specified in Appendix~\ref{exp_detail}.

\vspace{-2.3mm}
\subsection{Discriminative Learning}
\vspace{-2mm}

\setlength{\columnsep}{6pt}
\begin{wraptable}{r}[0cm]{0pt}
    \centering
    \scriptsize
    \newcommand{\tabincell}[2]{\begin{tabular}{@{}#1@{}}#2\end{tabular}}
    \setlength{\tabcolsep}{4pt}
\renewcommand{\captionlabelfont}{\scriptsize}
\begin{tabular}{c c c} 
\specialrule{0em}{-11pt}{0pt}
  \hline
Method  & Error\\\hline
Baseline  & 2.14\\
Orthogonal  & 1.95\\\hline
MHE & 1.85\\
MHS & 1.72\\
MHP  & 1.92\\
R-MHP & 1.99\\
MHC  & 1.88\\
MGD & \textbf{1.64}\\
 \hline
  \specialrule{0em}{0pt}{-10pt}
\end{tabular}
\caption{\scriptsize MLP (\%).}
\label{mlp}
\vspace{-4mm}
\end{wraptable}

\textbf{Multi-layer perceptrons}. We first compare all the hyperspherical uniformity regularizations on MNIST with a 3-layer MLP. We use Xavier initialization~\cite{glorot2010understanding} to initialize neuron weights in MLP. Results (error rates) in Table~\ref{mlp} show that all the proposed methods outperform the baseline by a considerable margin, indicating that hyperspherical uniformity is generally useful to improve the generalization for MLP. Among all the proposed regularizations, MGD works the best and reduces the error of the baseline by more than 23\%.

\textbf{Convolutional neural networks}. We also compare all the hyperspherical uniformity regularizations on both plain VGG-like CNN~\cite{simonyan2014very} and ResNet~\cite{he2016deep}. Specific architecture configurations are in Appendix~\ref{exp_detail}.

\setlength{\columnsep}{4pt}
\begin{wraptable}{r}[0cm]{0pt}
    \centering
    \scriptsize
    \newcommand{\tabincell}[2]{\begin{tabular}{@{}#1@{}}#2\end{tabular}}
    \setlength{\tabcolsep}{4pt}
\renewcommand{\captionlabelfont}{\scriptsize}
\begin{tabular}{c c c} 
\specialrule{0em}{-10pt}{0pt}
  \hline
Method  & CNN-9 & ResNet-32\\\hline
Baseline  & 28.13 & 22.87\\
Orthogonal  & 26.94 & 22.36\\
SRIP~\cite{bansal2018can} & 25.92 & 22.02\\\hline
MHE & 25.94 & 21.82\\
MHS & 25.43 & \textbf{20.97}\\
MHP  & 25.92 & 21.24\\
R-MHP & 26.02 & 22.19\\
MHC  & 25.62 & 21.88\\
MGD & \textbf{25.32} & 21.06\\
 \hline
  \specialrule{0em}{0pt}{-10pt}
\end{tabular}
\caption{\scriptsize CNN on CIFAR-100 (\%).}
\label{cnn_cifar}
\vspace{-1.8mm}
\end{wraptable}

\emph{CIFAR-100}. We first perform experiments on CIFAR-100 with the plain CNN-9 and ResNet-32. The experimental settings mostly follows \cite{liu2018learning} for fair comparison. Table~\ref{cnn_cifar} gives the error rates. The results show that most of the hyperspherical uniformity regularizations perform significantly and consistently better than the baseline, the orthogonal regularization and a state-of-the-art orthogonality-based regularization called SRIP~\cite{bansal2018can}. Surprisingly, R-MHP, as a very simple regularization method, also outperforms both the baseline and orthogonal regularization, implying the effectiveness of hyperspherical uniformity in general. Although MHP and MHC are less stable to optimize with gradient descent, they still perform reasonably well compared to the baseline, implying that promoting hyperspherical uniformity is generally beneficial to generalization. Mostly notably, MGD performs the best on CNN-9, while MHS performs the best on ResNet-18. It partially suggests that encouraging local separation may be more important to generalization than promoting global uniformity. One of the possible reasons could be that MHE updates all the neurons simultaneously by accumulating all the pairwise interactions while MHS only focus on updating the two most similar neurons.

\begin{figure}[t]
\vspace{0.1mm}
  \centering
  \renewcommand{\captionlabelfont}{\scriptsize}
  \setlength{\abovecaptionskip}{0.9pt}
  \setlength{\belowcaptionskip}{-12.5pt}
\includegraphics[width=3.35in]{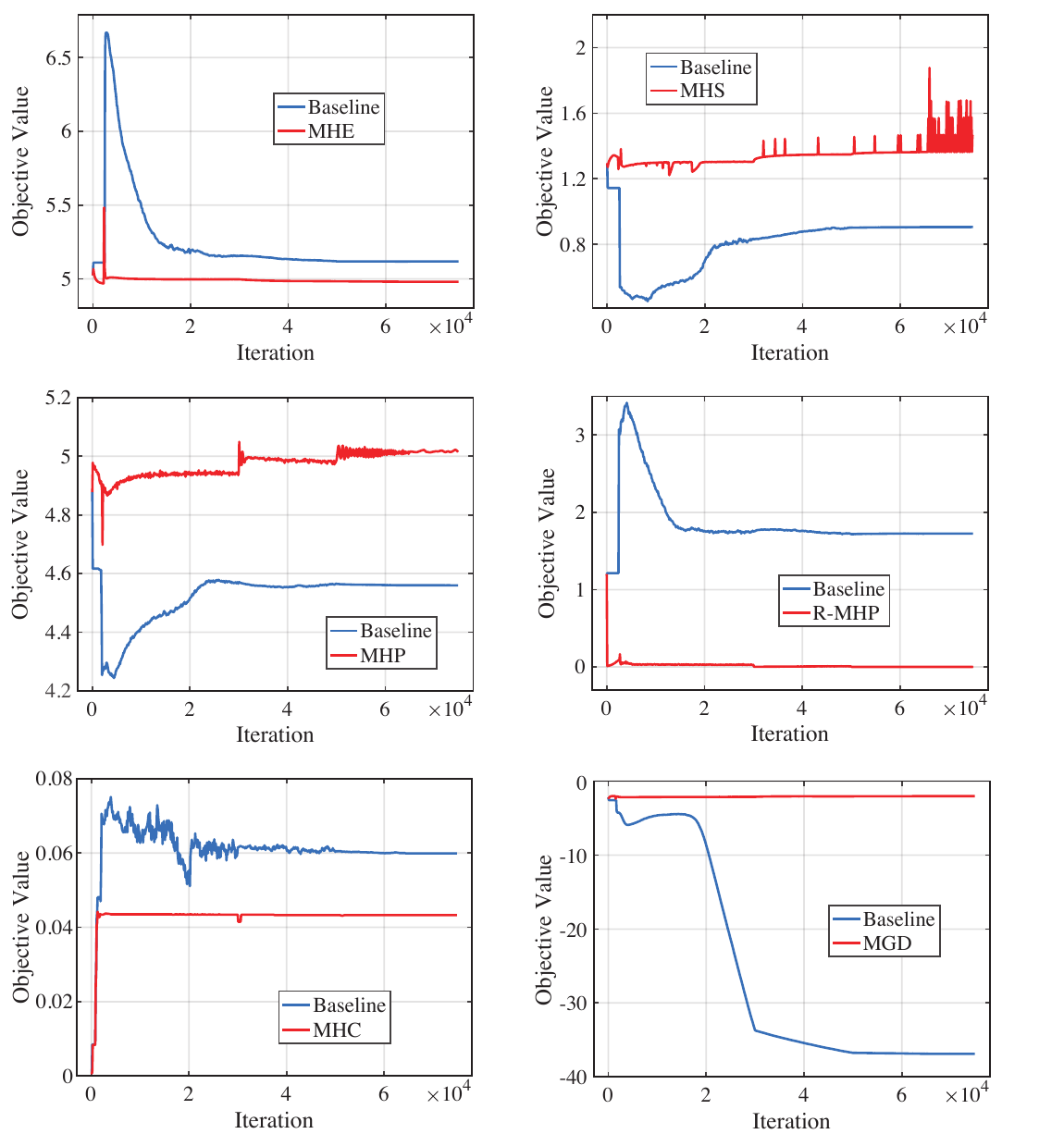}
  \caption{\scriptsize Regularization dynamics (objective value vs. iteration) of MHE, MHS, MHP, R-MHP, MHC and MGD.}\label{obj_value}
\end{figure}

\emph{Regularization dynamics}. Then we look into the effectiveness of the hyperspherical uniformity regularizations for minimizing (or maximizing) the objective values during training. We perform classification on CIFAR-100 with plain CNN-9, and plot the objective value curve of each hyperspherical uniformity regularization during training. The results are given in Fig.~\ref{obj_value}. Note that, MHS, MHP and MGD maximize the regularization objective, while MHE, R-MHP and MHC minimize the objective. From Fig.~\ref{obj_value}, we can see that all the hyperspherical uniformity achieve better objective values than the baseline, which indicates that all the proposed regularizations have well served the purpose.

\setlength{\columnsep}{6pt}
\begin{wraptable}{r}[0cm]{0pt}
    \centering
    \scriptsize
    \newcommand{\tabincell}[2]{\begin{tabular}{@{}#1@{}}#2\end{tabular}}
    \setlength{\tabcolsep}{5pt}
\renewcommand{\captionlabelfont}{\scriptsize}
\begin{tabular}{c c c} 
\specialrule{0em}{-10pt}{0pt}
  \hline
Method  & Error\\\hline
Baseline  & 32.95\\
Orthogonal  & 32.65\\
SRIP~\cite{bansal2018can} & 32.53\\\hline
MHE & 32.45\\
MHS & \textbf{32.06}\\
MHP  & 32.32\\
R-MHP & 32.71\\
MHC  & 32.28\\
MGD & 32.16\\
 \hline
  \specialrule{0em}{0pt}{-10pt}
\end{tabular}
\caption{\scriptsize ImageNet (\%).}
\label{imagenet}
\vspace{-2mm}
\end{wraptable}

\emph{ImageNet}. Finally, we conduct the experiments on ImageNet-2012~\cite{russakovsky2014imagenet} with ResNet-18 to further evaluate the performance of hyperspherical uniformity on large-scale datasets. Since our purpose is to compare all the regularizations, we use the same simple data augmentation scheme as in AlexNet~\cite{krizhevsky2012imagenet}. Detailed settings are given in Appendix~\ref{exp_detail}. Table~\ref{imagenet} shows the Top-1 error on ImageNet test set. One can see from the results that most of the hyperspherical uniformity regularizations can effectively improve the generalization of ResNet-18 on ImageNet. Among all, MHS achieves the best performance and outperforms the baseline by nearly 1\%. Considering the fact that we are merely adding a regularization without changing the network architecture, the improvement is actually very significant.

\setlength{\columnsep}{5pt}
\begin{wraptable}{r}[0cm]{0pt}
    \centering
    \scriptsize
    \newcommand{\tabincell}[2]{\begin{tabular}{@{}#1@{}}#2\end{tabular}}
    \setlength{\tabcolsep}{3pt}
\renewcommand{\captionlabelfont}{\scriptsize}
\begin{tabular}{c c c c} 
\specialrule{0em}{-12.5pt}{0pt}
\hline
Method &  Citeseer & Cora & Pubmed \\\hline
Baseline & 70.3 & 81.3 & 79.0\\
Orthogonal & 70.4 & 81.5 & 78.8\\\hline
MHE &  71.5 & 82.0 & 79.0\\
MHS & 71.7 & \textbf{82.3} & \textbf{79.2}\\
MHP  & 71.3 & 81.5 & 79.0\\
MHC  & 71.2 & 81.6 & 79.0\\
MGD &  \textbf{71.8} & \textbf{82.3} & \textbf{79.2}\\
 \hline
  \specialrule{0em}{0pt}{-10pt}
\end{tabular}
\caption{\scriptsize Graph networks (\%).}
\label{graph_net}
\vspace{-2mm}
\end{wraptable}

\textbf{Graph networks}. In order to show that hyperspherical uniformity is diversely useful, we also perform experiments on graph convolution networks (GCN) \cite{kipf2016semi}. We use the same 2-layer GCN as \cite{kipf2016semi} and evaluate on Citeseer, Cora and Pubmed data \cite{sen2008collective}.  Regularizing GCN is conceptually similar to MLP. Specifically, the forward model of GCN is $\thickmuskip=2mu \medmuskip=2mu\bm{Z}=\textnormal{Softmax}\big(\hat{\bm{A}}\cdot\textnormal{ReLU}(\hat{\bm{A}}\cdot\bm{X}\cdot\bm{W}_0)\cdot\bm{W}_1\big)$ where $\thickmuskip=2mu \medmuskip=2mu \hat{\bm{A}}=\tilde{\bm{D}}^{\frac{1}{2}}\tilde{\bm{A}}\tilde{\bm{D}}^{\frac{1}{2}}$. $\bm{A}$ is the adjacency matrix of the graph,  $\thickmuskip=2mu \medmuskip=2mu \tilde{\bm{A}}=\bm{A}+\bm{I}$ ($\bm{I}$ is an identity matrix), and $\thickmuskip=2mu \medmuskip=2mu \tilde{\bm{D}}=\sum_j\tilde{\bm{A}}_{ij}$. $\bm{X}\in\mathbb{R}^{n\times d}$ is the feature matrix of $n$ nodes in the graph (feature dimension is $d$). $\bm{W}_1$ is the weights of the classifiers. $\bm{W}_0$ is the weight matrix of size $d\times h$ where $h$ is the dimension of the hidden space. We treat each column vector of $\bm{W}_0$ as a neuron, so there are $h$ neurons in total. We apply the hyperspherical uniformity regularizations on $\bm{W}_0$ and report the testing accuracy in Table~\ref{graph_net}. From the results, we observe all the hyperspherical uniformity regularizations outperform the baseline by a considerable margin, indicating that promoting hyperspherical uniformity is very helpful to the generalization of GCN. On Citeseer, MGD achieves the best accuracy and outperforms the baseline by 1.5\%. On Cora and Pubmed, both MHS and MGD perform the best. The results validate the universality and superiority of hyperspherical uniformity.

\setlength{\columnsep}{6pt}
\begin{wraptable}{r}[0cm]{0pt}
    \centering
    \scriptsize
    \newcommand{\tabincell}[2]{\begin{tabular}{@{}#1@{}}#2\end{tabular}}
    \setlength{\tabcolsep}{4pt}
\renewcommand{\captionlabelfont}{\scriptsize}
\begin{tabular}{c c c} 
\specialrule{0em}{-12.5pt}{0pt}
  \hline
Method  & Accuracy\\\hline
Baseline  & 87.10\\
MHE & 87.44\\
MHS & 87.60\\
MHP  & 87.41\\
R-MHP & 87.10\\
MHC  & 87.33\\
MGD & \textbf{87.61}\\
 \hline
  \specialrule{0em}{0pt}{-10pt}
\end{tabular}
\caption{\scriptsize PointNet (\%).}
\label{pointnet}
\vspace{-4mm}
\end{wraptable}

\textbf{Point cloud networks}. We also evaluate the hyperspherical uniformity on the 3D point cloud classification task where each 3D object is represented by a unordered set of points (\ie, 3D coordinates). PointNet~\cite{qi2017pointnet} is a neural network designed for processing point clouds. PointNet consists of a group of weight-sharing MLPs and we regularize these MLPs using our hyperspherical uniformity regularizations. For simplicity, we use a vanilla PointNet without T-Net and experiment on the ModelNet-40 dataset~\cite{wu20153d}. The classification accuracy is given in Table~\ref{pointnet}. All the hyperspherical uniformity regularizations consistently outperform the baseline PointNet by a significant margin, showing the effectiveness of hyperspherical uniformity in point cloud classification. Among all the compared regularizations, MGD achieves the best accuracy by outperforming the baseline by 0.51\%, validating its universality in improving generalization for different types of neural networks.

\vspace{-2.3mm}
\subsection{Generative Modeling}
\vspace{-2.2mm}

Beside the applications in the discriminative learning for classification tasks, we further apply hyperspherical uniformity to improve the generative adversarial network (GAN)~\cite{goodfellow2014generative} in unconditional image generation.

\setlength{\columnsep}{6pt}
\begin{wraptable}{r}[0cm]{0pt}
    \centering
    \scriptsize
    \newcommand{\tabincell}[2]{\begin{tabular}{@{}#1@{}}#2\end{tabular}}
    \setlength{\tabcolsep}{3.5pt}
\renewcommand{\captionlabelfont}{\scriptsize}
\begin{tabular}{c c c} 
\specialrule{0em}{-10pt}{0pt}
  \hline
Method  & Inception Score \\\hline
Baseline  & 7.14 \\
SN~\cite{miyato2018spectral} & 7.40\\\hline
MHE & 7.40 \\
MHS & \textbf{7.61} \\
R-MHP & 7.31 \\
MGD & 7.49 \\
 \hline
  \specialrule{0em}{0pt}{-10pt}
\end{tabular}
\caption{\scriptsize GAN on CIFAR-10.}\label{GAN_tab}
\label{gan}
\vspace{-4mm}
\end{wraptable}

Specifically, we perform image generation on CIFAR-10 with a vanilla GAN. We regularize the discriminator in the vanilla GAN using MHE, MHS, MGD and R-MHP. We adopt the following vanilla GAN as the clean baseline: $V(G, D) =  \mathbb{E}_{\bm{x}\sim q_{\rm data}(\bm{x})} [ \log D(\bm{x})] +  \mathbb{E}_{\bm{z}\sim p(\bm{z})} [\log(1-D(G(\bm{z})))]$ where $D(\cdot)$ denotes the discriminator and $G(\cdot)$ denotes the generator. We do not use spectral normalization for the discriminator and generator. We train the models for 200k iterations on CIFAR-10 using Adam optimizer ($\beta_1 = 0.5, \beta_2 = 0.999$) with learning rate set to 1e-3 and batch size set to 64. Detailed experimental setup is given in Appendix~\ref{exp_detail}. Results in Table~\ref{GAN_tab} show that hyperspherical uniformity can generally improve the generation quality of GANs. Notably, MHS achieves the best inception score 7.61, which outperforms the vanilla GAN and the state-of-the-art spectral normalization (SN)~\cite{miyato2018spectral} by a considerable margin. MGD also performs better than SN, well verifying that promoting hyperspherical uniformity is beneficial to GANs.

\begin{figure}[h]
\vspace{-1mm}
  \centering
  \renewcommand{\captionlabelfont}{\scriptsize}
  \setlength{\abovecaptionskip}{-0.1pt}
  \setlength{\belowcaptionskip}{-6pt}
\includegraphics[width=3in]{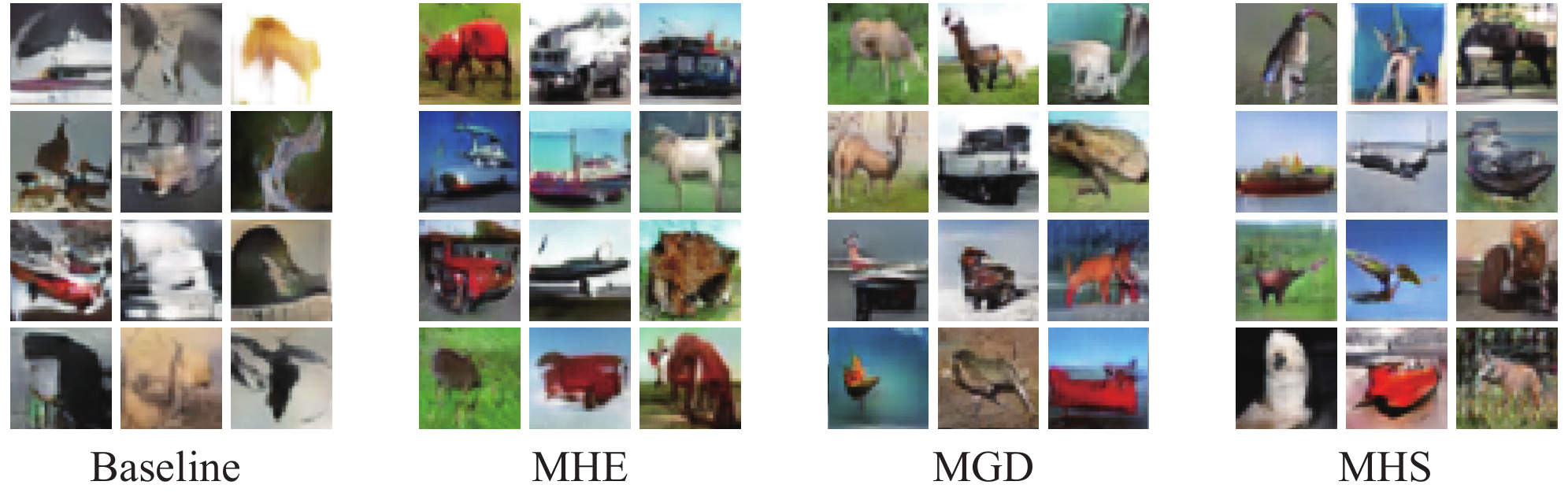}
\vspace{-0.2mm}
  \caption{\scriptsize Comparison of randomly generated images from GAN.}\label{GAN_images}
\end{figure}

We also give some qualitative examples in Fig.~\ref{GAN_images} to compare the generation quality between baseline and hyperspherical uniformity regularizations. The results show that hyperspherical uniformity can greatly improve GAN by regularizing it to generate visually plausible images with more diverse color and semantic meaning.

%\subsection{Reinforcement Learning}
%Will include experiment on DQN on Atari Games
\vspace{-2.6mm}
\section{Concluding Remarks}
\vspace{-2.9mm}

This paper considers a novel family of regularizations -- hyperspherical uniformity for training neural networks. Specifically, we propose several conceptually appealing instances and provide some statistical discussions and theoretical insights. Comprehensive experiments well validate the effectiveness of hyperspherical uniformity.

\newpage

\subsection*{Acknowledgements}

\vspace{-1mm}

Weiyang Liu and Adrian Weller acknowledge support from the Leverhulme Trust via CFI. Adrian Weller acknowledges support from the David MacKay Newton research fellowship at Darwin College, and The Alan Turing Institute under EPSRC grant EP/N510129/1 and U/B/000074. Rongmei Lin and Li Xiong are supported by NSF under CNS-1952192, IIS-1838200.

\bibliographystyle{plain}
\bibliography{ref}

\newpage
\onecolumn
\begin{appendix}
\begin{center}
{\Large \textbf{Appendix}}
\end{center}
\section{Proof of Proposition~\ref{mhs}}

We first define that $\hat{\bm{W}}_n^s$ is a $s$-energy minimizing $n$-point configuration on $\mathbb{S}^{d-1}$ if $0<s<\infty$ (\ie, MHE configuration) and $\hat{\bm{W}}_n^\infty$ denotes a best-packing configuration on $\mathbb{S}^{d-1}$ if $s=\infty$ (\ie, MHS configuration). Since we are considering $s>0$, we only need to discuss the case of $K_s(\hat{\bm{w}}_i,\hat{\bm{w}}_j)=\rho(\hat{\bm{w}}_i,\hat{\bm{w}}_j)^{-s}$. Then we will have the following equation:
\begin{equation}\label{p1_eq1}
    \varepsilon_s(\mathbb{S}^{d-1},n)^{\frac{1}{s}}=E_s(\hat{\bm{W}}_n^s)^{\frac{1}{s}}\geq\frac{1}{\delta^\rho_n(\hat{\bm{W}}_n^s)}\geq\frac{1}{\delta^\rho_n(\mathbb{S}^{d-1})}.
\end{equation}
Moreover, we have that
\begin{equation}
\begin{aligned}
\varepsilon_s(\mathbb{S}^{d-1},n)^{\frac{1}{s}}&\leq E_s(\hat{\bm{W}}^\infty_n)^{\frac{1}{s}}\\
&=\frac{1}{\delta^\rho(\hat{\bm{W}}^\infty_n)}\bigg( \sum_{1\leq i\neq j \leq N} \big( \frac{\delta^\rho(\hat{\bm{W}}^\infty_n)}{\rho(\hat{\bm{w}}_i^\infty,\hat{\bm{w}}_j^\infty)} \big)^s \bigg)^{\frac{1}{s}}\\
&\leq \frac{1}{\delta^\rho(\hat{\bm{W}}^\infty_n)}\big(n(n-1)\big)^{\frac{1}{s}}
\end{aligned}
\end{equation}
Therefore, we will end up with
\begin{equation}\label{p1_eq3}
    \lim_{s\rightarrow \infty}\sup\varepsilon_s(\mathbb{S}^{d-1},n)^{\frac{1}{s}}\leq\frac{1}{\delta^\rho(\hat{\bm{W}}_n^\infty)}=\frac{1}{\delta^\rho_n(\mathbb{S}^{d-1})}.
\end{equation}
Then we take both Eq.~\eqref{p1_eq1} and Eq.~\eqref{p1_eq3} into consideration and have that
\begin{equation}
    \lim_{s\rightarrow\infty}\varepsilon_s(\mathbb{S}^{d-1},n)^{\frac{1}{s}}=\frac{1}{\delta_n^{\rho}(\mathbb{S}^{d-1})}
\end{equation}
which concludes the proof.\hfill\qedsymbol

\section{Proof of Proposition~\ref{mhp}}
We first choose $\epsilon>0$ and let $\hat{\bm{W}}_{n+1}=\{\hat{\bm{w}}_1,\hat{\bm{w}}_2,\cdots,\hat{\bm{w}}_{n+1}\}\subset\mathbb{S}^{d-1}$ be a configuration such that
\begin{equation}
    \varepsilon_s(\mathbb{S}^{d-1},n+1)+\epsilon>E_s(\hat{\bm{W}}_{n+1}).
\end{equation}
Then we have for every $i\in[1,n+1]$ and $\bm{v}\in\mathbb{S}^{d-1}$ that
\begin{equation}
    \begin{aligned}
    E_s\big((\hat{\bm{W}}_{n+1}\backslash\{\hat{\bm{w}}_i\} )\cup\{\bm{v}\} \big)&=E_s(\hat{\bm{W}}_{n+1}\backslash\{\hat{\bm{w}}_i\})+2\sum_{j:j\neq i}K_s(\bm{v},\hat{\bm{w}}_j)\\
    &\geq\varepsilon_s(\mathbb{S}^{d-1},n+1)\\
    &>E_s(\hat{\bm{W}}_{n+1})-\epsilon\\
    &=E_s(\hat{\bm{W}}_{n+1}\backslash\{\bm{v}\})+2\sum_{j:j\neq i}K_s(\hat{\bm{w}}_i,\hat{\bm{w}}_j)-\epsilon
    \end{aligned}
\end{equation}
which leads to
\begin{equation}
    \min_{\bm{v}\in\mathbb{S}^{d-1}}2\sum_{j:j\neq i}K_s(\bm{v},\hat{\bm{w}}_j)\geq2\sum_{j:j\neq i}K_s(\hat{\bm{w}}_i,\hat{\bm{w}}_j)-\epsilon
\end{equation}
Therefore, for a fixed $i$, we have that
\begin{equation}
\begin{aligned}
    \mathcal{P}_s(\mathbb{S}^{d-1},n)&\geq P_s(\hat{\bm{W}}_{n+1}\backslash\{\hat{\bm{w}}_i\})\\
    &=\min_{\bm{v}\in\mathbb{S}^{d-1}}\sum_{j:j\neq i}K_s(\bm{v},\hat{\bm{w}}_j)\\
    &\geq\sum_{j:j\neq i}K_s(\hat{\bm{w}}_i,\hat{\bm{w}}_j)-\frac{\epsilon}{2}
\end{aligned}
\end{equation}
Then we average the above inequalities for $i=1,\cdots,n+1$ and obtain
\begin{equation}
\begin{aligned}
    \mathcal{P}_s(\mathbb{S}^{d-1},n)&\geq\frac{1}{n+1}\sum_{i=1}^{n+1}\sum_{j:j\neq i} K_s(\hat{\bm{w}}_i,\hat{\bm{w}}_j)-\frac{\epsilon}{2}\\
    &\geq \frac{\varepsilon_s(\mathbb{S}^{d-1},n+1)}{n+1}-\frac{\epsilon}{2}
\end{aligned}
\end{equation}
By letting $\epsilon$ approach to zero, we have that
\begin{equation}
    \mathcal{P}_s(\mathbb{S}^{d-1},n)\geq\frac{\varepsilon_s(\mathbb{S}^{d-1},n+1)}{n+1}
\end{equation}

Moreover, it is also easy to verify another inequality: 
\begin{equation}
    \frac{\varepsilon_s(\mathbb{S}^{d-1},n+1)}{n+1}\geq\frac{\varepsilon_s(\mathbb{S}^{d-1},n)}{n-1}
\end{equation}
Therefore, we conclude the proof.\hfill\qedsymbol

\section{Proof of Proposition~\ref{mhp_relax}}

Given that $s=-2$, we first have that
\begin{equation}
\begin{aligned}
P_{-2}(\hat{\bm{W}}_n)&=\min_{\bm{v}\in\mathbb{S}^{d-1}}\bigg( -\sum_{i=1}^n\norm{\bm{v}-\hat{\bm{w}}_i}^2 \bigg)\\
&=\min_{\bm{v}\in\mathbb{S}^{d-1}}\sum_{i=1}^n(2\bm{v}\cdot\hat{\bm{w}}_i-2)\\
&=\min_{\bm{v}\in\mathbb{S}^{d-1}}\bigg( 2\bm{v}\cdot\sum_{i=1}^n\hat{\bm{w}}_i -2n \bigg).
\end{aligned}
\end{equation}
If $\sum_{i=1}^n\hat{\bm{w}}_i=\bm{0}$, we will have that $P_{-2}(\hat{\bm{W}}_n)=-2n$. If $\sum_{i=1}^n\hat{\bm{w}}_i\neq \bm{0}$, then we have that
\begin{equation}
    \begin{aligned}
    P_{-2}(\hat{\bm{W}}_n)&\leq -2\frac{\sum_{i=1}^n\hat{\bm{w}}_i}{\norm{\sum_{i=1}^n\hat{\bm{w}}_i}}\cdot\sum_{i=1}^n\hat{\bm{w}}_i-2n\\
    &=-2\norm{\sum_{i=1}^n\hat{\bm{w}}_i}-2n\\
    &<-2n
    \end{aligned}
\end{equation}
Therefore, $\hat{\bm{W}}_n$ is optimal if and only if $\sum_{i=1}^n\hat{\bm{w}}_i=\bm{0}$\hfill\qedsymbol

\section{Proof of Proposition~\ref{mhc}}
For any $n$-point configuration $\hat{\bm{W}}_n\subset\mathbb{S}^{d-1}$, we have that
\begin{equation}
\begin{aligned}
    P_s(\hat{\bm{W}}_n)&=\min_{\bm{v}\in\mathbb{S}^{d-1}}\sum_{\bm{u}\in\hat{\bm{W}}_n}\frac{1}{\rho(\bm{v},\bm{u})^s}\\
    &\geq\frac{1}{\alpha(\hat{\bm{W}}_n)^s}
\end{aligned}
\end{equation}
which leads to
\begin{equation}
    \begin{aligned}
    \big(\mathcal{P}_s(\mathbb{S}^{d-1},n)\big)^{\frac{1}{s}}&=\max_{\hat{\bm{W}}_n\subset\mathbb{S}^{d-1}}P_s(\hat{\bm{W}}_n)^{\frac{1}{s}}\\
    &\geq\max_{\hat{\bm{W}}_n\subset\mathbb{S}^{d-1}}\frac{1}{\alpha(\hat{\bm{W}}_n)}\\
    &=\frac{1}{\eta_n^\rho(\mathbb{S}^{d-1})}.
    \end{aligned}
\end{equation}
Therefore, we have that
\begin{equation}\label{mhc_eq1}
    \lim_{s\rightarrow\infty}\inf\big(\mathcal{P}_s(\mathbb{S}^{d-1},n)\big)^{\frac{1}{s}}\geq\frac{1}{\eta_n^\rho(\mathbb{S}^{d-1})}
\end{equation}
On the other hand, we have that
\begin{equation}
    \begin{aligned}
    P_s(\hat{\bm{W}}_n)&=\min_{\bm{v}\in\mathbb{S}^{d-1}}\sum_{\bm{u}\in\hat{\bm{W}}_n}\frac{1}{\rho(\bm{v},\bm{u})^s}\\
    &\leq \frac{n}{\alpha(\hat{\bm{W}}_n)}\\
    &\leq\frac{n}{\big(\eta_n^\rho(\mathbb{S}^{d-1})\big)^s}
    \end{aligned}
\end{equation}
Therefore, we end up with
\begin{equation}
    \lim_{s\rightarrow\infty}\sup\big(\mathcal{P}_s(\mathbb{S}^{d-1},n)\big)^{\frac{1}{s}}\leq\lim_{s\rightarrow\infty}\frac{n^\frac{1}{s}}{\eta_n^\rho(\mathbb{S}^{d-1})} =\frac{1}{\eta_n^\rho(\mathbb{S}^{d-1})}
\end{equation}
Combining with Eq.~\eqref{mhc_eq1}, we have that
\begin{equation}
    \lim_{s\rightarrow\infty}\big(\mathcal{P}_s(\mathbb{S}^{d-1},n)\big)^{\frac{1}{s}}\geq\frac{1}{\eta_n^\rho(\mathbb{S}^{d-1})}
\end{equation}
which concludes the proof.\hfill\qedsymbol

\section{Proof of Proposition~\ref{range_test}}
We first define the order samples on $\mathbb{S}^1$. We denote the samples on $\mathbb{S}^1$ as $\theta_i$. The angles are ordered such that $\theta_{i+1}<\theta_i,\forall i$. Then we define the angle gap as follows:
\begin{equation}
\begin{aligned}
d_i&:=\theta_{i+1}-\theta_i,\ \ i=1,2,\cdots,n-1\\
d_n&:=2\pi-(\theta_n-\theta_1)
\end{aligned}
\end{equation}
The test statistic of range test is written as
\begin{equation}
    T_n:=2\pi-\max_i d_i
\end{equation}
which rejects $\mathcal{H}_0$ for small values. Maximizing $T_n$ with respect to the samples on $\mathbb{S}^1$ is equivalent to the following objective:
\begin{equation}
    \max_{\{\theta_1,\cdots,\theta_n\}} T_n \Leftrightarrow  \min_{\{\theta_1,\cdots,\theta_n\}}\max_{i}d_i
\end{equation}
which is to minimize the largest neighbor angle gap. It is easy to verify that the optimum happens when the $n$ angle gaps are equally divided the unit circle $\mathbb{S}^1$.
\par
For MHS on $\mathbb{S}^1$, the optimization is as follows:
\begin{equation}
    \max_{\{\theta_1,\cdots,\theta_n\}}\min_{i\neq j}\rho(\theta_i,\theta_j)
\end{equation}
which is to maximize the smallest pairwise angles (\ie, the smallest neighbor angle gap on $\mathbb{S}^1$). The optimum is attained when $\{\theta_1,\cdots,\theta_n\}$ are equally divided the unit circle $\mathbb{S}^1$, which is equivalent to maximizing $T_n$ with respect to the samples on $\mathbb{S}^1$.
\par
For MHC on $\mathbb{S}^1$, the optimization is as follows:
\begin{equation}
    \min_{\{\theta_1,\cdots,\theta_n\}}\max_{v\in[0,2\pi)}\min_i\rho(v,\theta_i).
\end{equation}
The optimum of $\max_{v\in[0,2\pi)}\min_i\rho(v,\theta_i)$ is attained when $v$ lies on the middle point of the largest angle gap. Therefore, the optimum of MHC on $\mathbb{S}^1$ is achieved when $\{\theta_1,\cdots,\theta_n\}$ are equally divided the unit circle $\mathbb{S}^1$, which is also equivalent to maximizing $T_n$ with respect to the samples on $\mathbb{S}^1$.\hfill\qedsymbol

\newpage
\section{Proof of Theorem~\ref{optimum_1}}
We first let $\hat{\bm{W}}_n=\{\hat{\bm{w}}_1,\cdots,\hat{\bm{w}}_n\}$ be an arbitrary vector configuration in $\mathbb{S}^d$. We then have that
\begin{equation}
\begin{aligned}
    \Lambda(\hat{\bm{W}}_n):=&\sum_{i=1}^n\sum_{j=1}^n \norm{\hat{\bm{w}}_i-\hat{\bm{w}}_j}^2\\
    =& \sum_{i=1}^n\sum_{j=1}^n (2-2\hat{\bm{w}}_i\cdot\hat{\bm{w}}_j)\\
    =& 2n^2-2\norm{\sum_{i=1}^n\hat{\bm{w}}_i}^2\\
    \leq & 2n^2
\end{aligned}
\end{equation}
which holds if and only if $\sum_{i=1}^n\hat{\bm{w}}_i=0$. The vertices of a regular $(n-1)$-simplex at the origin well satisfy this condition. With the properties of the potential function $f$, we have that
\begin{equation}
\begin{aligned}
    E_f(\hat{\bm{W}}_n):=&\sum_{i=1}^n\sum_{j:j\neq i} f\big( \norm{\hat{\bm{w}}_i-\hat{\bm{w}}_j}^2\big)\\
    \geq & n(n-1) f\bigg( \frac{\Lambda(\hat{\bm{W}}_n)}{n(n-1)} \bigg)\\
    \geq & n(n-1)f\bigg( \frac{2n}{n-1} \bigg)
\end{aligned}
\end{equation}
which holds true if all pairwise distance $\norm{\hat{\bm{w}}_i-\hat{\bm{w}}_j}$ are equal for $i\neq j$ and the center of mass is at the origin (\ie, $\sum_{i=1}^n\hat{\bm{w}}_i=\bm{0}$). Therefore, for the vector configuration $\hat{\bm{W}}_n^*$ which contains the vertices of a regular $(n-1)$-simplex inscribed in $\mathbb{S}^d$ and centered at the origin, we have that for $2\leq n \leq d+2$
\begin{equation}
\begin{aligned}
    E_f(\hat{\bm{W}}_n^*)&=n(n-1)f\bigg( \frac{2n}{n-1} \bigg)\\
    &\leq E_f(\hat{\bm{W}}_n).
\end{aligned}
\end{equation}
If $f$ is strictly convex and strictly decreasing, then $E_f(\hat{\bm{W}}_n)\geq n(n-1)f( \frac{2n}{n-1} )$ holds only when $\hat{\bm{W}}_n^*$ is a regular $(n-1)$-simplex inscribed in $\mathbb{S}^d$ and centered at the origin.\hfill\qedsymbol

\newpage
\section{Proof of Theorem~\ref{lower_MHE}}

Let $\hat{\bm{w}}_1^*,\hat{\bm{w}}_2^*,\cdots,\hat{\bm{w}}_n^*$ be the points in the MHE solution $\hat{\bm{W}}_n^*$. Without loss of generality, we denote the indices $k$ and $l$ such that $\vartheta(\hat{\bm{W}}_n^*)=\|\hat{\bm{w}}_k^*-\hat{\bm{w}}_l^*\|_2$. We also define $\bm{z}:=(1+n^{-\frac{1}{d-1}})\hat{\bm{w}}_k^*$. We first introduce the following fact about closed convex sets:
\begin{proposition}
Let $\bm{K}\subset\mathbb{R}^p$ be a closed convex set. Then for every $\bm{x}\in\mathbb{R}^p$, there is a unique point $\bm{y}_x$ in $\bm{K}$ closest to $\bm{x}$. Furthermore, for any $\bm{z}\in\bm{K}$, we have $\|\bm{y}_x\bm{z}\|_2\leq\|\bm{x}-\bm{z}\|_2$, where the equality holds if and only if $\bm{x}\in\bm{K}$.
\end{proposition}

Because the unit hyperball $B(\bm{0},1)$ is convex and $\hat{\bm{w}}_k^*$ is the point in $B(\bm{0},1)$ closest to $\bm{z}$, for $1\leq j\leq n$ we have the following inequality based on this proposition above:
\begin{equation}
    \|\hat{\bm{w}}_k^*-\hat{\bm{w}}_j^*\|_2\leq\|\bm{z}-\hat{\bm{w}}_j^*\|_2,
\end{equation}
where $1\leq j\leq n$. Before we proceed, we need to introduce the following lemmas:

\begin{lemma}[\cite{kuijlaars2007separation}]\label{mhe_lemma1}
If $0<s<d-1$ and $\hat{\bm{W}}_n^*=\{\hat{\bm{w}}_1^*,\cdots,\hat{\bm{w}}_n^*\}$ is a MHE solution on $\mathbb{S}^{d-1}$, then for $i=1,2,\cdots,n$, we have that
\begin{equation}
    \frac{1}{n-1}\sum_{j:j\neq i} \frac{1}{\|\hat{\bm{w}}_i^*-\hat{\bm{w}}_j^*\|_2^s}\leq I_s[\sigma_{d-1}]
\end{equation}
where $I_s[\mu]=\int\int\frac{1}{\|\bm{x}-\bm{y}\|_2^s}d\mu(\bm{x})d\mu(\bm{y})$ and $\sigma_{d-1}$ is the normalized probability surface area measure on $\mathbb{S}^{d-1}$.
\end{lemma}

\begin{lemma}[\cite{kuijlaars2007separation}]\label{mhe_lemma2}
We assume $d-2\leq s<d-1$, and then there is a constant $\theta_{s,d}$ and a positive integer $m$ such that for every $\bm{x}\in\mathbb{R}^{d}$ with $\|\bm{x}\|_2=1+n^{-\frac{1}{d-1}}$ and any optimal MHE solution $\hat{\bm{W}}_n^*$ on $\mathbb{S}^{d-1}$, we have
\begin{equation}
    U_s(\bm{x};\hat{\bm{W}}_n^*)\geq I_s[\sigma_{d-1}]-\theta_{s,d}\cdot n^{-1+\frac{s}{d-1}}
\end{equation}
where $n>m$ and $U_s(\bm{x};\hat{\bm{W}}_n^*) :=\frac{1}{n}\sum_{\bm{y}\in\hat{\bm{W}}_n^*}\frac{1}{\|\bm{x}-\bm{y}\|_2^s}$ for $s>0$.
\end{lemma}

Using Lemma~\ref{mhe_lemma1} above, we obtain that

\begin{equation}
\begin{aligned}
    I_s[\sigma_{d-1}]-\frac{1}{n\vartheta(\hat{\bm{W}}_n^*)^s}&\geq \frac{1}{n}\big( \sum_{j:j\neq k} \frac{1}{\| \hat{\bm{w}}_k^* - \hat{\bm{w}}_j^* \|_2^s} -\frac{1}{\|\hat{\bm{w}}_k^*-\hat{\bm{w}}_l^*\|^s_2} \big) \\
    &=\frac{1}{n}\sum_{j:j\neq k,l} \frac{1}{\|\hat{\bm{w}}_k^*-\hat{\bm{w}}_j^*\|_2^s}\\
    &\geq\frac{1}{n}\sum_{j:j\neq k,l}\frac{1}{\|\bm{z}-\hat{\bm{w}}_j^*\|_2^s}\\
    &=U_s(\bm{z};\hat{\bm{W}}_n^*)-\frac{1}{n}\big( \frac{1}{\|\bm{z}-\hat{\bm{w}}_k^*\|_2^s}+\frac{1}{\|\bm{z}-\hat{\bm{w}}_l^*\|_2^s} \big).
\end{aligned}
\end{equation}
Because of $n^{-\frac{1}{d-1}}=\|\bm{z}-\hat{\bm{w}}_k^*\|_2^s\leq\|\bm{z}-\hat{\bm{w}}_l^*\|_2^s$, we have that
\begin{equation}
    I_s[\sigma_{d-1}]-\frac{1}{n\vartheta(\hat{\bm{W}}_n^*)^s}\geq U_s(\bm{z};\hat{\bm{W}}_n^*)-2n^{-1+\frac{s}{d-1}}
\end{equation}
Then according to Lemma~\ref{mhe_lemma2}, we have that
\begin{equation}
    I_s[\sigma_{d-1}]-\frac{1}{n\vartheta(\hat{\bm{W}}_n^*)^s}\geq I_s[\sigma_{d-1}]-(\theta_{s,d}+2)\cdot n^{-1+\frac{s}{d-1}}
\end{equation}
which concludes that $\vartheta(\hat{\bm{W}}_n^*)\geq\lambda_{s,d}\cdot n^{-\frac{1}{d-1}}$
where we define that $\lambda_{s,d}=(\theta_{s,d}+2)^{-\frac{1}{s}}$. Note that, the extended and generalized version of this result can be found in \cite{dragnev2007riesz,brauchart2014riesz,kuijlaars2007separation}. \hfill\qedsymbol

\section{Proof of Theorem~\ref{energy_MHS}}
The theorem comes directly from the result in \cite{leopardi2013discrepancy} that every asymptotically optimal MHS sequence $\{\hat{\bm{W}}_n^*\}_{n=2}^\infty$ of $n$-point configurations on $\mathbb{S}^{d-1}$ is asymptotically optimal MHE solution for any $0<s<d-1$.

\section{Proof of Theorem~\ref{spectral}}
We first introduce the following lemma as the characterization of a unit vector that is uniformly distributed on the unit hypersphere $\mathbb{S}^{d-1}$.
\begin{lemma}[\cite{o2016eigenvectors}]
Let $\bm{v}$ be a random vector that is uniformly distributed on the unit hypersphere $\mathbb{S}^{d-1}$. Then $\bm{v}$ has the same distribution as the following:
\begin{equation}
    \bigg{\{}\frac{u_1}{\sqrt{\sum_{i=1}^du_i^2}},\frac{u_2}{\sqrt{\sum_{i=1}^du_i^2}},\cdots,\frac{u_d}{\sqrt{\sum_{i=1}^du_i^2}}\bigg{\}}
\end{equation}
where $u_1,u_2,\cdots,u_d$ are \emph{i.i.d.} standard normal random variables.
\end{lemma}
\begin{proof}
The lemma follows naturally from the fact that the Gaussian vector $\{u_i\}_{i=1}^d$ is rotationally invariant.
\end{proof}

Then we consider a random matrix $\tilde{\bm{W}}=\{\tilde{\bm{v}}_1,\cdots,\tilde{\bm{v}}_n\}$ where $\tilde{\bm{v}}_i$ follows the same distribution of $\{u_1,\cdots,u_d\}$. Therefore, it is also equivalent to a random matrix with each element distributed normally. For such a matrix $\tilde{\bm{W}}$, we have from \cite{silverstein1985smallest} that
\begin{equation}\label{sinval_bound}
\begin{aligned}
    \lim_{n\rightarrow\infty}\sigma_{\max}(\tilde{\bm{W}})&=\sqrt{d}+\sqrt{\lambda d}\\
    \lim_{n\rightarrow\infty}\sigma_{\min}(\tilde{\bm{W}})&=\sqrt{d}-\sqrt{\lambda d}
\end{aligned}
\end{equation}
where $\sigma_{\max}(\cdot)$ and $\sigma_{\min}(\cdot)$ denote the largest and the smallest singular value, respectively. 
\par
Then we write the matrix $\bm{W}$ as follows:
\begin{equation}
\begin{aligned}
    \bm{W}&=\tilde{\bm{W}}\cdot\bm{Q}\\
    &=\tilde{\bm{W}}\cdot\begin{bmatrix}
    \frac{1}{\|\tilde{\bm{v}}_1\|_2}& 0 &\cdots& 0\\
    0 & \frac{1}{\|\tilde{\bm{v}}_2\|_2} & \ddots & 0\\
    \vdots & \ddots & \ddots  & \vdots  \\
    0 & \cdots & 0 & \frac{1}{\|\tilde{\bm{v}}_n\|_2}
    \end{bmatrix}
\end{aligned}
\end{equation}
which leads to 
\begin{equation}\label{lim_ineq}
\begin{aligned}
    \lim_{n\rightarrow\infty}\sigma_{\max}(\bm{W})&=\lim_{n\rightarrow\infty}\sigma_{\max}(\tilde{\bm{W}}\cdot\bm{Q})\\
    \lim_{n\rightarrow\infty}\sigma_{\min}(\bm{W})&=\lim_{n\rightarrow\infty}\sigma_{\min}(\tilde{\bm{W}}\cdot\bm{Q})
\end{aligned}.
\end{equation}
\par
We fist assume that for a symmetric matrix $\bm{A}\in\mathbb{R}^{n\times n}$ $\lambda_1(\bm{A})\geq\cdots\geq\lambda_n(\bm{A})$. Then we introduce the following inequalities for eigenvalues:
\begin{lemma}[\cite{marshall1979inequalities}]\label{eigenvalue_ineq}
Let $\bm{G},\bm{H}\in\mathbb{R}^{n\times n}$ be positive semi-definite symmetric, and let $1\leq i_1<\cdots<i_k\leq n$. Then we have that
\begin{equation}
    \prod_{t=1}^k\lambda_{i_t}(\bm{G}\bm{H})\leq\prod_{t=1}^k\lambda_{i_t}(\bm{G})\lambda_t(\bm{H})
\end{equation}
and
\begin{equation}
    \prod_{t=1}^k \lambda_{i_t}(\bm{G}\bm{H})\geq \prod_{t=1}^k\lambda_{i_t}(\bm{G})\lambda_{n-t+1}(\bm{H})
\end{equation}
where $\lambda_i$ denotes the $i$-th largest eigenvalue.
\end{lemma}

We first let $1\leq i_1<\cdots<i_k\leq n$. Because $\tilde{\bm{W}}\in\mathbb{R}^{d\times n}$ and $\bm{Q}\in\mathbb{R}^{n\times n}$, we have the following:
\begin{equation}
\begin{aligned}
    \prod_{t=1}^k\sigma_{i_t}(\tilde{\bm{W}}\bm{Q})&=\prod_{t=1}^k\sqrt{\lambda_{i_t}(\tilde{\bm{W}}\bm{Q}\bm{Q}^\top\tilde{\bm{W}}^\top)}\\
    &=\sqrt{\prod_{t=1}^k \lambda_{i_t}(\tilde{\bm{W}}^\top\tilde{\bm{W}}\bm{Q}\bm{Q}^\top)}
\end{aligned}
\end{equation}
by applying Lemma~\ref{eigenvalue_ineq} to the above equation, we have that
\begin{equation}
\begin{aligned}
    \sqrt{\prod_{t=1}^k \lambda_{i_t}(\tilde{\bm{W}}^\top\tilde{\bm{W}}\bm{Q}\bm{Q}^\top)}&\geq \sqrt{\prod_{t=1}^k \lambda_{i_t}(\tilde{\bm{W}}^\top\tilde{\bm{W}})\lambda_{n-t+1}(\bm{Q}\bm{Q}^\top)}\\
    &= \prod_{t=1}^k\sigma_{i_t}(\tilde{\bm{W}})\sigma_{n-t+1}(\bm{Q})
\end{aligned}
\end{equation}
\begin{equation}
\begin{aligned}
    \sqrt{\prod_{t=1}^k \lambda_{i_t}(\tilde{\bm{W}}^\top\tilde{\bm{W}}\bm{Q}\bm{Q}^\top)}&\leq \sqrt{\prod_{t=1}^k \lambda_{i_t}(\tilde{\bm{W}}^\top\tilde{\bm{W}})\lambda_{t}(\bm{Q}\bm{Q}^\top)}\\
    &= \prod_{t=1}^k\sigma_{i_t}(\tilde{\bm{W}})\sigma_{t}(\bm{Q})
\end{aligned}
\end{equation}
Therefore, we have that
\begin{equation}\label{sinval_geq}
    \prod_{t=1}^k\sigma_{i_t}(\tilde{\bm{W}}\bm{Q})\geq\prod_{t=1}^k\sigma_{i_t}(\tilde{\bm{W}})\sigma_{n-t+1}(\bm{Q})
\end{equation}
\begin{equation}\label{sinval_leq}
    \prod_{t=1}^k\sigma_{i_t}(\tilde{\bm{W}}\bm{Q})\leq\prod_{t=1}^k\sigma_{i_t}(\tilde{\bm{W}})\sigma_{t}(\bm{Q})
\end{equation}
Suppose we have $k=1$ and $i_1=n$, then Eq.~\eqref{sinval_geq} gives
\begin{equation}
    \sigma_n(\tilde{\bm{W}}\bm{Q})\geq\sigma_n(\tilde{\bm{W}})\sigma_n(\bm{Q})
\end{equation}
Then suppose we have $k=1$ and $i_1=1$, then Eq.~\eqref{sinval_leq} gives
\begin{equation}
    \sigma_1(\tilde{\bm{W}}\bm{Q})\leq\sigma_1(\tilde{\bm{W}})\sigma_1(\bm{Q})
\end{equation}
Combining the above results with Eq.~\eqref{sinval_bound} and Eq.~\eqref{lim_ineq}, we have that
\begin{equation}
\begin{aligned}
    \lim_{n\rightarrow\infty}\sigma_{\max}(\bm{W})&=\lim_{n\rightarrow\infty} \sigma_{\max}(\tilde{\bm{W}}\cdot\bm{Q})\leq\lim_{n\rightarrow\infty}\big(\sigma_{\max}(\tilde{\bm{W}})\cdot\sigma_{\max}(\bm{Q})\big)=(\sqrt{d}+\sqrt{\lambda d})\cdot\max_i\frac{1}{\|\tilde{\bm{v}}_i\|_1}\\
    \lim_{n\rightarrow\infty}\sigma_{\min}(\bm{W})&=\lim_{n\rightarrow\infty}\sigma_{\min}(\tilde{\bm{W}}\cdot\bm{Q})\geq\lim_{n\rightarrow\infty}\big(\sigma_{\min}(\tilde{\bm{W}})\cdot\sigma_{\min}(\bm{Q})\big)=(\sqrt{d}-\sqrt{\lambda d})\cdot\min_i\frac{1}{\|\tilde{\bm{v}}_i\|_1}
\end{aligned}
\end{equation}
which concludes the proof.\hfill\qedsymbol

\newpage
\section{Hyperspherical Uniformity from Zero-mean Gaussian Distributions}\label{uniform_sphere}

We show that zero-mean equal-variance Gaussian distributed vectors (after normalized to norm $1$) are uniformly distributed over the unit hypersphere with the following theorem.

\begin{theorem}\label{sphereuniform}
The normalized vector of Gaussian variables is uniformly distributed on the sphere. Formally, let $x_1,x_2,\cdots,x_n\sim \mathcal{N}(0,1)$ and be independent. Then the vector
\begin{equation}
    \bm{x}=\bigg{[} \frac{x_1}{z},\frac{x_2}{z},\cdots,\frac{x_n}{z} \bigg{]}
\end{equation}
follows the uniform distribution on $\mathbb{S}^{n-1}$, where $z=\sqrt{x_1^2+x_2^2+\cdots+x_n^2}$ is a normalization factor.
\end{theorem}
\begin{proof}
A random variable has distribution $\mathcal{N}(0,1)$ if it has the density function
\begin{equation}
    f(x)=\frac{1}{\sqrt{2\pi}}e^{-\frac{1}{2}x^2}.
\end{equation}
A $n$-dimensional random vector $\bm{x}$ has distribution $\mathcal{N}(0,1)$ if the components are independent and have distribution $\mathcal{N}(0,1)$ each. Then the density of $\bm{x}$ is given by
\begin{equation}
    f(x)=\frac{1}{(\sqrt{2\pi})^n}e^{-\frac{1}{2}\langle x,x\rangle}.
\end{equation}
Then we introduce the following lemma (Lemma~\ref{lemma_sphereuniform}) about the orthogonal-invariance of the normal distribution.
\begin{lemma}\label{lemma_sphereuniform}
Let $\bm{x}$ be a $n$-dimensional random vector with distribution $\mathcal{N}(0,1)$ and $\bm{U}\in\mathbb{R}^{n\times n}$ be an orthogonal matrix ($\bm{U}\bm{U}^\top =\bm{U}^\top\bm{U}=\bm{I} $). Then $\bm{Y}=\bm{U}\bm{x}$ also has the distribution of $\mathcal{N}(0,1)$.
\end{lemma}
\begin{proof}
For any measurable set $A\subset\mathbb{R}^n$, we have that
\begin{equation}
\begin{aligned}
    P(Y\in A)&= P(X\in U^\top A)\\
    &=\int_{U^\top A}\frac{1}{(\sqrt{2\pi})^n} e^{-\frac{1}{2}\langle x,x \rangle}\\
    &=\int_A\frac{1}{(\sqrt{2\pi})^n}e^{-\frac{1}{2}\langle Ux, Ux \rangle}\\
    &=\int_A\frac{1}{(\sqrt{2\pi})^n}e^{-\frac{1}{2}\langle x, x \rangle}
\end{aligned}
\end{equation}
because of orthogonality of $U$. Therefore the lemma holds.
\end{proof}

Because any rotation is just a multiplication with some orthogonal matrix, we know that normally distributed random vectors are invariant to rotation. As a result, generating $\bm{x}\in\mathbb{R}^n$ with distribution $\mathbb{N}(0,1)$ and then projecting it onto the hypersphere $\mathbb{S}^{n-1}$ produces random vectors $U=\frac{\bm{x}}{\|\bm{x}\|}$ that are uniformly distributed on the hypersphere. Therefore the theorem holds.
\end{proof}

\section{Orthogonality vs. Orthonormality}

In the paper, we sometimes use the term ``orthogonality'' and ``orthonormality'' interchangeably, since we are mostly considering the points lying in $\mathbb{S}^d$. Hyperspherical uniformity only concerns with the angles among points (\eg, neurons), because all the magnitude are normalized to one before entering any hyperspherical uniformity objective. Therefore strictly speaking, orthogonality is a more appropriate comparison to hyperspherical uniformity.

\newpage
\section{Experimental Details}\label{exp_detail}

\begin{table}[h]
    \renewcommand{\captionlabelfont}{\scriptsize}
    \newcommand{\tabincell}[2]{\begin{tabular}{@{}#1@{}}#2\end{tabular}}
    \centering
    \setlength{\abovecaptionskip}{5pt}
    \setlength{\belowcaptionskip}{0pt}
    \scriptsize
    \begin{tabular}{|c|c|c|c|c|}
        \hline
        Layer & CNN-9 for CIFAR-100 & ResNet-32 for CIFAR-100 & ResNet-18 for ImageNet-2012 \\
        \hline\hline
        Conv0.x & N/A & [3$\times$3, 64]  & \tabincell{c}{[7$\times$7, 64], Stride 2\\3$\times$3, Max Pooling, Stride 2} \\\hline
        Conv1.x & \tabincell{c}{[3$\times$3, 64]$\times$3\\2$\times$2 Max Pooling, S2} & $\left[\begin{aligned} &3\times 3, 64\\&3\times3, 64\end{aligned}\right]\times 5$ & $\left[\begin{aligned} &3\times 3, 64\\&3\times3, 64\end{aligned}\right]\times 2$ \\ \hline
        Conv2.x  & \tabincell{c}{[3$\times$3, 128]$\times$3\\2$\times$2 Max Pooling, S2} & $\left[\begin{aligned} &3\times 3, 128\\&3\times3, 128\end{aligned}\right]\times 5$  & $\left[\begin{aligned} &3\times 3, 128\\&3\times3, 128\end{aligned}\right]\times 2$ \\\hline
        Conv3.x  & \tabincell{c}{[3$\times$3, 256]$\times$3\\2$\times$2 Max Pooling, S2} & $\left[\begin{aligned} &3\times 3, 256\\&3\times3, 256\end{aligned}\right]\times 5$  & $\left[\begin{aligned} &3\times 3, 256\\&3\times3, 256\end{aligned}\right]\times 2$  \\\hline
        Conv4.x & N/A & N/A & $\left[\begin{aligned} &3\times 3, 512\\&3\times3, 512\end{aligned}\right]\times 2$  \\\hline
        Final & 256-Dim Fully Connected & \multicolumn{2}{c|}{Average Pooling}  \\\hline
    \end{tabular}
    \caption{\scriptsize  Our CNN and ResNet architectures with different convolutional layers. Conv0.x, Conv1.x, Conv2.x, Conv3.x and Conv4.x denote convolution units that may contain multiple convolutional layers, and residual units are shown in double-column brackets. Conv1.x, Conv2.x and Conv3.x usually operate on different size feature maps. These networks are essentially similar to \cite{simonyan2014very} and \cite{he2016deep}, but with different number of filters in each layer. The downsampling is performed by convolutions with a stride of 2. E.g., [3$\times$3, 64]$\times$4 denotes 4 cascaded convolution layers with 64 filters of size 3$\times$3, and S2 denotes stride 2.}\label{cnn_netarch}
\end{table}

\begin{table}[h]
\setlength{\abovecaptionskip}{2pt}
\setlength{\belowcaptionskip}{5pt}
\renewcommand{\captionlabelfont}{\scriptsize}
    \centering
    \begin{subtable}{.49\linewidth}
        \centering
        \scriptsize
        {\begin{tabular}{|c|}
            \hline
            $z\in \mathbb{R}^{128} \sim \mathcal{N}(0, I)$ \\
            \hline
            dense $\rightarrow$ $M_g$ $\times$ $M_g$ $\times$ 512 \\
            \hline
            4$\times$4, stride=2 deconv. BN 256 ReLU\\
            \hline
            4$\times$4, stride=2 deconv. BN 128 ReLU\\
            \hline
            4$\times$4, stride=2 deconv. BN 64 ReLU\\
            \hline
            3$\times$3, stride=1 conv. 3 Tanh\\
            \hline
        \end{tabular}}
        \caption{\scriptsize \label{tab:gen}Generator ($M_g=4$  for CIFAR10).}
    \end{subtable}
    \begin{subtable}{.49\linewidth}
        \centering
        \scriptsize
        {\begin{tabular}{|c|}
            \hline
            RGB image $x\in \mathbb{R}^{M\times M \times 3}$ \\
            \hline
            3$\times$3, stride=1 conv 64 lReLU\\
            4$\times$4, stride=2 conv 64 lReLU\\
            \hline
            3$\times$3, stride=1 conv 128 lReLU\\
            4$\times$4, stride=2 conv 128 lReLU\\
            \hline
            3$\times$3, stride=1 conv 256 lReLU\\
            4$\times$4, stride=2 conv 256 lReLU\\
            \hline
            3$\times$3, stride=1 conv. 512 lReLU\\
            \hline
            dense $\rightarrow$ 1 \\
            \hline
        \end{tabular}}
        \caption{\scriptsize \label{tab:dis_deep}Discriminator ($M=32$ CIFAR10).}
    \end{subtable}
    \caption{\scriptsize \label{gan_netarch} GAN architecture for image generation on CIFAR-10. The architectures mostly follow \cite{miyato2018spectral}.}
\end{table}

\textbf{General}. For MHE, we use half-space MHE with $s=2$. For the unrolling of MHP and MHC, we use one-step gradient descent to approximate the inner optimization. For MHC, we use the relaxed formulation with $\gamma=5$.  For MGD, we use Gaussian kernel with $\epsilon=1$. Typically, we search the best weighting hyperparameter for all the regularizations from $10^{-8}$ to $10^{7}$ (with 10 as the step size).

\textbf{Multilayer perceptron}. We conduct hand-written digit recognition task on MNIST with a three-layer multilayer perceptron following this repository\footnote{\url{https://github.com/hwalsuklee/tensorflow-mnist-MLP-batch_normalization-weight_initializers}} . The size of each digit image is $28\times 28$, which is 784 dimensions after flattened. Both hidden layers have 256 output dimensions, \ie, 256 neurons. The output layer will output 10 logits for classification. Finally, we use a cross-entropy loss with softmax function. We use the momentum SGD optimizer with learning rate 0.01, momentum 0.9 and batch size 100. The training stops at 100 epochs. 

\textbf{Convolutional neural networks}. The network architectures used in the main paper are specified in Table~\ref{cnn_netarch}. For all experiments, we use the momentum SGD optimizer with momentum 0.9. For CIFAR-100, we set the mini-batch size as 128. The learning rate starts at 0.1, and is divided by 10 when the performance is saturated. For ImageNet-2012, we use the mini-batch size 128 and the training starts with learning rate 0.1. The learning rate is divided by 10 when the performance is saturated, and the training is terminated at 700k iterations. The structure of ResNet-18 mostly follows \cite{he2016deep}. Note that, for all the methods in our experiments, we always use the best possible hyperparameters for the corresponding regularization (via cross-validation) to make sure that the comparison is fair. The baseline has exactly the same training settings as the others. Standard $\ell_2$ weight decay ($\thickmuskip=2mu \medmuskip=2mu 5e-4$) is applied by default to all the methods. 

\textbf{Graph networks}. We implement the all the hyperspherical uniformity regularizations for GCN in the official repository\footnote{\url{https://github.com/tkipf/gcn}}. All the hyperparameter settings exactly follow this official repository to ensure a fair comparison. 

\textbf{Point cloud networks}. To simplify the comparison and remove all the bells and whistles, we use a vanilla PointNet (without T-Net) as our backbone network. We apply OPT to train the MLPs in PointNet. We follow the same experimental settings as \cite{qi2017pointnet} and evaluate on the ModelNet-40 dataset~\cite{wu20153d}. We exactly follow the same setting in the original paper~\cite{qi2017pointnet} and the official repositories\footnote{\url{https://github.com/charlesq34/pointnet}}. Specifically, we use the hyperspherical uniformity regularizations to regularize all the $\thickmuskip=2mu \medmuskip=2mu 1\times1$ convolution layers and the fully connected layer (except the final classifier). For the experiments, we set the point number as 1024 and mini-batch size as 32. We use the Adam optimizer with initial learning rate 0.001. The learning rate will decay by 0.7 every 200k iterations, and the training is terminated at 250 epochs.

\textbf{Generative adversarial networks}. The architecture we use for the GAN experiments is shown in Table~\ref{gan_netarch}. For fair comparison, all the hyperparameter settings exactly follow \cite{miyato2018spectral}. We use leaky ReLU (LReLU) in the newtork and set the slopes of LReLU functions to $0.1$.

\subsection{Experimental details for Fig.~\ref{separation_exp}}\label{exp_detail_fig1}

For the experiment in Fig.~\ref{separation_exp}, we use 200 3-dimensional neurons. The momentum SGD optimizer (with momentum 0.9) is used to optimize these hyperspherical uniformity objectives. The learning rate starts at 0.01 and is divided by 10 at 5k iterations. The optimization stops at 8k iterations. In Fig.~\ref{separation_exp}(a), the y-axis denotes the value of hyperspherical energy. In Fig.~\ref{separation_exp}(b), the y-axis denotes the value of separation distance. We did not visualize MHP, since the true objective value of MHP is difficult (also time-consuming) to compute. For MHE, we use half-space MHE with $s=2$. For MHC, we use one-step gradient descent to approximate the inner optimization and also adopt the relaxed formulation with $\gamma=5$. For MGD, we use Gaussian kernel with $\epsilon=1$.

\subsection{Experimental details for Fig.~\ref{3d_sphere}}

For the visualization experiment in Fig.~\ref{3d_sphere}, we optimize 100 3-dimensional neurons on the unit sphere. The training hyperparameters are the same as Section~\ref{exp_detail_fig1}.

\end{appendix}

\end{document}